\documentclass[runningheads]{llncs}

\usepackage[T1]{fontenc}
\usepackage[utf8]{inputenc}
\usepackage{graphicx}
\usepackage{booktabs}
\usepackage{amsmath}
\usepackage{amssymb}
\usepackage{amsfonts}
\usepackage{nicefrac}
\usepackage[expansion=false]{microtype}
\usepackage{xcolor}
\usepackage{multirow}
\usepackage{tabularx}
\usepackage{adjustbox}
\usepackage{subcaption}
\usepackage{caption}
\usepackage{algorithm}
\usepackage{algorithmic}
\usepackage{cleveref}
\usepackage{enumitem}
\usepackage[most]{tcolorbox}
\usepackage{float}
\usepackage{newfloat}
\usepackage{listings}
\usepackage{textcomp}
\usepackage[hyphens]{url}
\usetikzlibrary{arrows.meta,positioning,fit,backgrounds}

\graphicspath{{../}{./}}

\spnewtheorem{nbtheorem}{Theorem}{\bfseries}{\itshape}
\spnewtheorem{nbproposition}{Proposition}{\bfseries}{\itshape}
\spnewtheorem{nbdefinition}{Definition}{\bfseries}{\itshape}

\newcommand{\Tau}{T}
\newcommand{\ModelName}{FITTER}
\newcommand{\parabold}[1]{\vspace{2pt}\noindent\textbf{#1}}

\floatstyle{ruled}
\newfloat{listing}{tb}{lst}{}
\floatname{listing}{Listing}

\begin{document}

\title{FITTER: Vocabulary-Agnostic Cross-Domain Inference on Temporal Knowledge Graphs}

\titlerunning{FITTER: Vocabulary-Agnostic Cross-Domain Inference on TKGs}
\author{Jiaxin Pan\inst{1} \and
Mojtaba Nayyeri\inst{2} \and
Osama Mohammed\inst{1} \and
Daniel Hern\'{a}ndez\inst{1} \and
Rongchuan Zhang\inst{1} \and
Cheng Cheng\inst{1} \and
Steffen Staab\inst{1,3}}
\authorrunning{J. Pan et al.}
\institute{University of Stuttgart, Stuttgart, Germany\\
\email{\{jiaxin.pan, osama.mohammed, steffen.staab\}@ki.uni-stuttgart.de}\\
\email{daniel@degu.cl}, \email{\{st191486, st180913\}@stud.uni-stuttgart.de}
\and
SAP SE, Germany\\
\email{mojtaba.nayyeri@ki.uni-stuttgart.de}
\and
University of Southampton, Southampton, United Kingdom}

\maketitle

\begin{abstract}
Temporal knowledge graphs are central to many uses of the Semantic Web, but existing completion methods assume the entities, relation names, and timestamps to be reasoned about are already known at training time, restricting each model to a single graph and vocabulary. We propose \ModelName{}, the first fully-inductive structural model for temporal knowledge graph link prediction that supports cross-domain transfer: the inference graph may contain entirely unseen entities, relation names, and timestamps drawn from a different domain. \ModelName{} represents each predicate by its interaction patterns with others and time through encodings of relative rather than absolute ordering; message-passing fuses local and global temporal context to produce vocabulary-agnostic embeddings. We prove the temporal encoding is time-shift invariant and evaluate \ModelName{} on cross-domain, cross-graph transfer over six temporal knowledge graph benchmarks of diverse domains, granularities, and time spans. \ModelName{} consistently outperforms inductive baselines without retraining, indicating that vocabulary-agnostic structural learning is a viable foundation for inference over the heterogeneous knowledge graphs of the Semantic Web.

\end{abstract}

\section{Introduction}
Temporal Knowledge Graphs (TKGs) are a central object of study in the Semantic Web: large openly published graphs such as Wikidata~\cite{vrandevcic2014wikidata}, DBpedia~\cite{lehmann2015dbpedia}, and YAGO~\cite{hoffart2013yago2} carry millions of facts that hold during specific intervals or at specific points in time, modelled through reified statements carrying temporal qualifiers (Wikidata's \textit{start time}, \textit{end time}, \textit{point in time}), through named graphs annotated with valid times, or through RDF-star triples. Abstracting over these representations, a TKG fact is commonly treated as a time-stamped quadruple $(s, p, o, \tau)$, where $s$, $p$, $o$, and $\tau$ denote the subject entity, predicate, object entity, and timestamp. Reasoning over such quadruples supports a wide range of downstream tasks including question answering~\cite{jia2021complex}, event forecasting~\cite{cai2024surveytemporalknowledgegraph}, and temporal recommendation; among these, temporal link prediction, i.e., predicting missing entities in queries of the form $(s, p, ?, \tau)$ or $(?, p, o, \tau)$, is one of the central problems in temporal knowledge graph research.

Existing Temporal Knowledge Graph Embedding (TKGE) models~\cite{leblay2018deriving,garcia-duran-etal-2018-learning,tcomplexlacroix2020tensor} are primarily developed for temporal interpolation or extrapolation settings, where the entity and relation vocabularies are shared between training and inference. In practice, however, temporal knowledge graphs often differ substantially across domains. For example, diplomatic event graphs such as ICEWS contain dense daily interactions, while encyclopedic graphs such as YAGO describe sparse long-term facts spanning decades or centuries. These datasets also differ in temporal granularity, time span, and relational structure.

This heterogeneity raises a natural transfer challenge: can a model trained on one TKG generalize to a different TKG without retraining? Such a capability is particularly useful in cold-start scenarios, where a target TKG is newly constructed or contains only limited training data. It is also desirable in practical deployments where training separate models for every temporal graph is computationally expensive. More broadly, despite differences in vocabulary and temporal scale, many TKGs share common structural and temporal patterns. For example, interaction sequences such as negotiation following conflict or recurring temporal transitions between related events may appear across multiple domains. A transferable model should therefore capture structural and temporal regularities independently of dataset-specific entities, relations, or timestamps.

Most existing TKGE models rely on dataset-specific entity, relation, and timestamp embeddings. Consequently, their learned representations cannot be directly applied to graphs containing unseen entities, relations, or timestamps. Recent work on static knowledge graphs has explored fully-inductive reasoning through structure-based relational learning \cite{galkintowards,lee2023ingram,zhangtrix,du2026graphoracle}, demonstrating that transferable representations can be learned without relying on fixed vocabularies. However, extending this paradigm to temporal knowledge graphs remains challenging due to the additional complexity introduced by temporal dynamics, heterogeneous temporal granularities, and varying time spans.

In this work, we study cross-domain transfer for temporal knowledge graphs under a fully-inductive setting, where the training and inference graphs contain disjoint entity, relation, and timestamp sets. We propose \textbf{\ModelName{}} (\underline{F}ully \underline{I}nductive \underline{T}ime-aware \underline{T}ransf\underline{e}rable \underline{R}epresentation), a structure-driven TKGE framework designed for cross-domain transfer across heterogeneous temporal knowledge graphs.

\ModelName{} addresses two key challenges in transferable temporal reasoning. First, temporal knowledge graphs may operate at different granularities and over different time spans, making dataset-specific timestamp embeddings difficult to transfer across domains. To address this, \ModelName{} represents temporal information using sinusoidal positional encodings over snapshot indices, enabling the model to capture relative temporal ordering independently of absolute timestamps. Second, to transfer structural knowledge across unseen vocabularies, \ModelName{} constructs universal relation interaction graphs based on vocabulary-agnostic interaction types and learns adaptive entity and relation representations through temporal-aware message passing. In addition, \ModelName{} combines local and global temporal contexts to capture both short-term event dependencies and long-range temporal patterns.

We evaluate \ModelName{} on 15 cross-dataset transfer settings spanning six temporal knowledge graphs with diverse domains, temporal granularities, and time spans. Experimental results show that \ModelName{} consistently outperforms existing inductive baselines in cross-domain transfer settings while remaining competitive with transductive approaches on in-domain evaluation.

Our contributions are summarized as follows:
\begin{itemize}
    \item We formulate cross-domain transfer for temporal knowledge graphs as a fully-inductive link prediction problem.
    \item We propose \ModelName{}, a transferable TKGE framework that avoids dataset-specific entity, relation, and timestamp embeddings.
    \item We introduce transferable temporal representations based on relative temporal ordering together with vocabulary-agnostic relational structure learning.
    \item We conduct extensive experiments across heterogeneous temporal knowledge graphs and demonstrate strong cross-domain transfer performance across domains and temporal granularities.
\end{itemize}

\section{Task Formulation}

\newcommand{\GG}{{G}}
\newcommand{\VV}{{V}}
\newcommand{\RR}{{R}}
\newcommand{\QQ}{\{Q\}}
\newcommand{\Gtrain}{{G}_{\mathrm{train}}}
\newcommand{\Vtrain}{{V}_{\mathrm{train}}}
\newcommand{\Rtrain}{{R}_{\mathrm{train}}}
\newcommand{\Ttrain}{T_{\mathrm{train}}}
\newcommand{\Qtrain}{{Q}_{\mathrm{train}}}
\newcommand{\Ginf}{{G}_{\mathrm{inf}}}
\newcommand{\Vinf}{{V}_{\mathrm{inf}}}
\newcommand{\Rinf}{{R}_{\mathrm{inf}}}
\newcommand{\Tinf}{T_{\mathrm{inf}}}
\newcommand{\Qinf}{{Q}_{\mathrm{inf}}}
\newcommand{\Gmsg}{{G}_{\mathrm{msg}}}
\newcommand{\Tmsg}{T_{\mathrm{msg}}}
\newcommand{\Qmsg}{{Q}_{\mathrm{msg}}}
\newcommand{\Gvalid}{{G}_{\mathrm{valid}}}
\newcommand{\Tvalid}{T_{\mathrm{valid}}}
\newcommand{\Qvalid}{{Q}_{\mathrm{valid}}}
\newcommand{\Gtest}{{G}_{\mathrm{test}}}
\newcommand{\Ttest}{T_{\mathrm{test}}}
\newcommand{\Qtest}{{Q}_{\mathrm{test}}}
\newcommand{\txt}{\Sigma^*}
\newcommand{\labl}{{l}}

A \emph{temporal knowledge graph} is a tuple $G = (V, R, T, Q)$, where $V$ is a set of entities, $R$ a set of relations, $T$ an ordered set of timestamps, and $Q \subseteq V \times R \times V \times T$ a set of temporal facts. The \emph{temporal link prediction} task asks: given a query $(s, p, ?, \tau)$ or $(?, p, o, \tau)$, predict the missing entity.

\begin{figure}[!t]
\centering
\includegraphics[width=\textwidth]{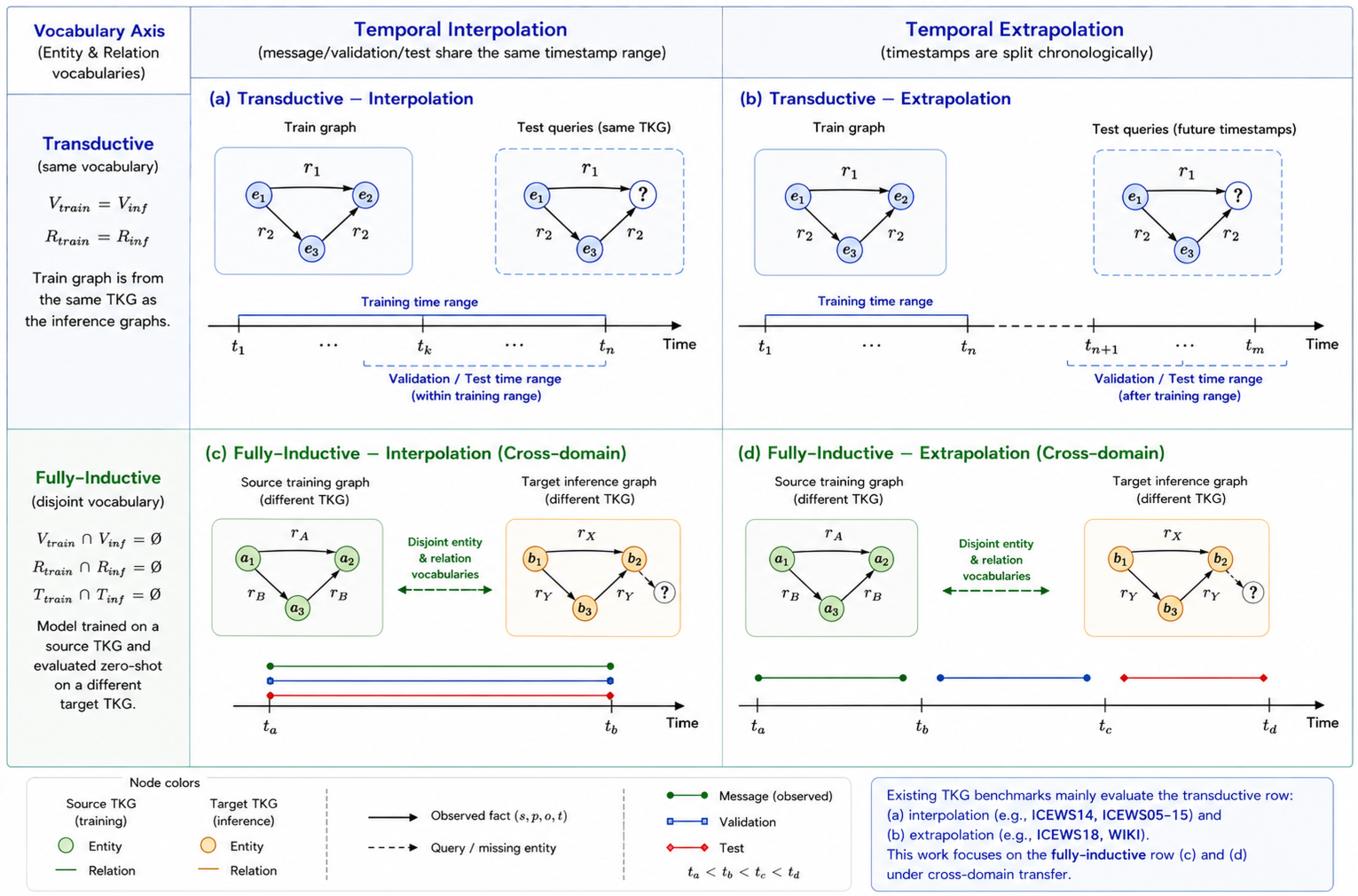}
\caption{Taxonomy of TKG inference settings along the vocabulary axis (transductive vs.\ fully-inductive) and temporal axis (interpolation vs.\ extrapolation). This work studies fully-inductive cross-domain transfer under both regimes.}
\label{fig:inference_example}
\end{figure}

\subsection{Inference Settings}
\label{sec:inference_settings}

We characterize TKG inference settings along two orthogonal axes: a \emph{vocabulary axis} and a \emph{temporal axis}. The vocabulary axis describes whether the entity and relation vocabularies are shared between the training and inference graphs. The temporal axis describes whether inference timestamps are evaluated within an observed temporal range (\emph{interpolation}) or on future timestamps beyond the observed range (\emph{extrapolation}). Combining the two axes yields a $2 \times 2$ taxonomy of inference settings, illustrated in Figure~\ref{fig:inference_example}.

Let $\Gtrain = (\Vtrain, \Rtrain, \Ttrain, \Qtrain)$ denote the \emph{source training graph} used for model training, and let $\Ginf = (\Vinf, \Rinf, \Tinf, \Qinf)$ denote the \emph{inference graph} used during evaluation. We partition the inference graph facts as
\[
    \Qinf = \Qmsg \cup \Qvalid \cup \Qtest,
\]
where $\Qmsg$, $\Qvalid$, and $\Qtest$ denote the \emph{message graph}, \emph{validation set}, and \emph{test set}, respectively. Their corresponding timestamp sets are $\Tmsg$, $\Tvalid$, and $\Ttest$.

\noindent\textbf{Vocabulary axis.}
In \emph{transductive} settings, the training and inference graphs share the same entity and relation vocabularies:
\[
    \Vtrain = \Vinf, \qquad \Rtrain = \Rinf.
\]
This is the standard setting adopted by existing TKG embedding methods, where models learn dataset-specific entity, relation, and timestamp embeddings.

In \emph{fully-inductive} settings, the training and inference graphs contain disjoint entity, relation, and timestamp sets:
\[
    \Vtrain \cap \Vinf = \emptyset, \qquad \Rtrain \cap \Rinf = \emptyset, \qquad \Ttrain \cap \Tinf = \emptyset.
\]
As a result, representations tied to a specific vocabulary are not directly applicable at inference time. A practically important special case is \textbf{cross-domain transfer}, where the \textbf{training and inference graphs originate from different source TKGs}.

\noindent\textbf{Temporal axis.}
In \emph{interpolation} settings, message, validation, and test facts share the same timestamp range:
\[
    \Tmsg = \Tvalid = \Ttest.
\]
In \emph{extrapolation} settings, the inference graph is split chronologically such that:
\[
    \max(\Tmsg) < \min(\Tvalid) < \min(\Ttest).
\]
For transductive settings, the message graph corresponds to the training graph of the same TKG. For fully-inductive settings, the message graph belongs to the target inference TKG and is used only for inference-time message passing.

\noindent\textbf{Our setting.}
Existing TKG benchmarks primarily evaluate models in transductive settings under either interpolation (e.g., ICEWS14, ICEWS05-15) or extrapolation (e.g., ICEWS18, WIKI). This paper studies fully-inductive inference under cross-domain transfer, where a model trained on one source TKG is evaluated directly on a different target TKG without retraining, under both interpolation and extrapolation conditions.

\begin{figure}[!t]
    \centering
    \includegraphics[width=\textwidth]{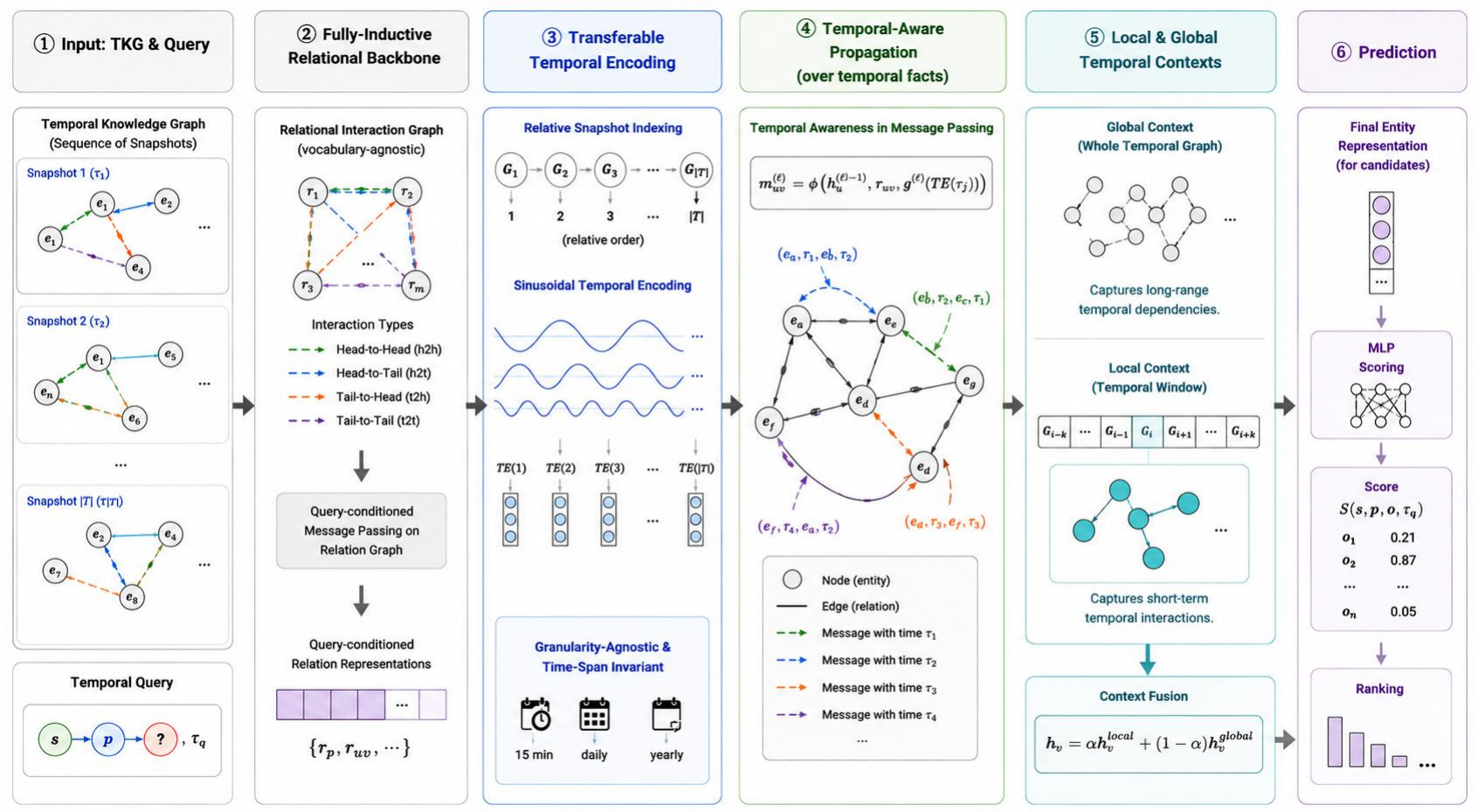}
    \caption{Overview of \ModelName{}. ULTRA constructs vocabulary-agnostic relation representations on a relation interaction graph. \ModelName{} adds three temporal components (blue): sinusoidal temporal encodings, temporal-aware entity propagation, and local/global context fusion. An MLP scores the fused representations.}
    \label{fig:model_overview}
\end{figure}

\section{Related Work}
\subsection{Temporal Knowledge Graph Embedding}

\noindent\textbf{Transductive temporal interpolation models.}
Early TKGE models primarily address temporal interpolation, where training and inference share the same entity and relation vocabularies and test queries are evaluated within the observed temporal range. TTransE~\cite{leblay2018deriving} and TA-DistMult~\cite{garcia-duran-etal-2018-learning} incorporate temporal information into scoring functions by treating time as an additional component of temporal facts. TComplEx and TNTComplEx~\cite{tcomplexlacroix2020tensor} formulate temporal knowledge graph completion as a fourth-order tensor completion problem, while TLT-KGE~\cite{tltcomplexzhang2022along} extends temporal representations into quaternion space. HGE~\cite{pan2024hge} further embeds temporal knowledge graphs into heterogeneous geometric subspaces and uses attention to capture structural and geometric similarities. Although effective in standard interpolation benchmarks, these methods rely on dataset-specific entity, relation, and timestamp embeddings, which prevents direct application to graphs with unseen vocabularies.

\noindent\textbf{Transductive temporal extrapolation models.}
Another line of work focuses on temporal extrapolation, where models predict future facts beyond the observed time range. Many of these approaches use recurrent or graph neural architectures to encode historical graph evolution and forecast future events~\cite{li2022complex,li2022tirgn,sun2021timetraveler,liang2023learn}. For improved explainability, TLogic~\cite{liu2022tlogic} extracts temporal logical rules via temporal random walks, while xERTE~\cite{han2020explainable} performs prediction using explainable temporal subgraphs.

\noindent\textbf{Limitations of existing TKGE models.}
Collectively, existing TKGE methods operate in settings (a) and (b) of Figure~\ref{fig:inference_example}: they assume vocabulary overlap between training and inference, making them unsuitable for the fully-inductive cross-domain settings (c) and (d) studied in this work. \ModelName{} addresses this gap with vocabulary-agnostic structural representations and transferable temporal encodings.

\subsection{Inductive and Cross-Domain Learning on KGs}

\noindent\textbf{Structure-based fully-inductive learning.}
Early inductive methods for static KGs, such as NBFNet~\cite{zhu2021neural}, Grail~\cite{teru2020inductive}, INDIGO~\cite{liu2021indigo}, and Morse~\cite{chen2022meta}, learn relational structures that generalize to unseen entities but still assume fixed relation vocabularies. More recent works, including INGRAM~\cite{lee2023ingram}, ULTRA~\cite{galkintowards}, TRIX~\cite{zhangtrix}, and GraphOracle~\cite{du2026graphoracle}, further enable transfer across entirely unseen entity and relation sets by constructing vocabulary-agnostic relation interaction graphs or relation dependency structures. However, all of these approaches are designed for static knowledge graphs and therefore do not model temporal dynamics, temporal ordering, or heterogeneous temporal granularities required for temporal knowledge graphs.

\noindent\textbf{Text-based and LLM-based transfer for TKGs.}
Recent approaches explore inductive inference on TKGs using textual descriptions and large language models (LLMs). Methods such as ICL~\cite{lee2023temporal}, zRLLM~\cite{ding2024zrllm}, GenTKG~\cite{liao2024gentkg}, and Chain-of-History~\cite{luo2024chain} leverage semantic information from entity descriptions, relation names, or in-context temporal reasoning to support transfer across unseen temporal facts. However, their transferability primarily arises from textual semantics rather than structural temporal representations. As a result, they depend heavily on annotation quality, textual availability, prompting strategies, and large-scale pretrained models, while often underutilizing the structural and relational dependencies inherent in temporal knowledge graphs. In contrast, \ModelName{} is purely \textit{structure-driven}: it learns transferable temporal and relational patterns directly from graph structure without requiring textual annotations or external language models.

\section{NBFNet-Style Fully-Inductive KG Reasoning}
\label{sec:background_nbfnet}
\label{sec:relation_learning}

We build on ULTRA~\cite{galkintowards}, a fully-inductive KG reasoning framework derived from NBFNet-style message passing~\cite{zhu2021neural}. ULTRA is particularly suitable for cross-domain transfer because it avoids dataset-specific entity and relation embeddings, enabling inference on entirely unseen graphs.

\noindent\textbf{Query-conditioned message passing.}
Given a query $(s, r, ?)$, these models perform query-conditioned reasoning by propagating information over the graph. The source entity $s$ is initialized using the query relation $r$, and message passing iteratively aggregates relational information from neighboring entities. At each layer $\ell$, messages are computed over graph edges and aggregated to update node representations:
\begin{equation}
    m_{uv}^{(\ell)} = \phi\!\left(h_u^{(\ell-1)},\, r_{uv}\right), \qquad
    h_v^{(\ell)} = \mathrm{AGG}\!\left(\left\{m_{uv}^{(\ell)} : u \in \mathcal{N}(v)\right\}\right).
\end{equation}

\noindent\textbf{Vocabulary-agnostic relation representations.}
Unlike conventional KG embedding methods that learn fixed embeddings for each relation, ULTRA constructs relation representations from a vocabulary-agnostic \emph{relation interaction graph}. In this graph, relations are treated as nodes and connected through structural interaction types: head-to-head, head-to-tail, tail-to-head, and tail-to-tail interactions. Message passing on the relation graph produces query-conditioned relation representations that generalize to unseen relations and unseen graphs.

\section{FITTER}

\ModelName{} introduces three temporal components on top of the relational backbone: (1) transferable temporal encodings that capture relative ordering independently of absolute timestamps; (2) temporal-aware message passing that incorporates time into fact propagation; and (3) local/global context fusion that integrates short-term and long-range temporal dependencies. Together, these components enable \ModelName{} to learn structural and temporal patterns that transfer across domains without any dataset-specific identifiers. The overall architecture is illustrated in Figure~\ref{fig:model_overview}.

\subsection{Transferable Temporal Encoding}
\label{sec:temporal_embedding}

\ModelName{} represents temporal information via \emph{relative snapshot ordering} rather than absolute timestamps, enabling transfer across TKGs with different granularities and time spans. Given a temporal knowledge graph $G$, we represent it as an ordered sequence of temporal snapshots $\{G_1, G_2, \dots, G_{|T|}\}$, where each snapshot $G_i$ contains all temporal facts associated with the $i$-th timestamp in chronological order. We write $\tau_i$ for the timestamp at snapshot $G_i$.

We encode the temporal position of snapshot $G_i$ using sinusoidal positional encodings:
\begin{align*}
    [\operatorname{TE}(i)]_{2n}    &= \sin(\omega_n\, i), \qquad
    [\operatorname{TE}(i)]_{2n+1}   = \cos(\omega_n\, i),\\
    \omega_n &= \beta^{-2n/d},
\end{align*}
where $d$ denotes the temporal embedding dimension and $\beta$ controls the frequency scale. The index $n \in \{0, 1, \dots, d/2 - 1\}$ enumerates the $d/2$ dimension pairs of the encoding: each pair $\big([\operatorname{TE}(i)]_{2n}, [\operatorname{TE}(i)]_{2n+1}\big)$ oscillates at its own frequency $\omega_n = \beta^{-2n/d}$, which decreases geometrically with $n$, so that small $n$ captures fast, short-term variation and large $n$ captures slow, long-range variation. We also write $\operatorname{TE}_{\tau_i} = \operatorname{TE}(i)$.

Unlike learned timestamp embeddings, this representation is independent of dataset-specific timestamp identities and depends only on relative temporal ordering. As a result, it naturally supports transfer across temporal knowledge graphs with different timestamp vocabularies, temporal granularities, and time spans. Furthermore, the multi-frequency sinusoidal representation enables the model to capture both short-term and long-range temporal dependencies, as supported by prior analysis of sinusoidal encodings in Transformer architectures~\cite{su2024roformer}.

\subsection{Temporal-Aware Quadruple Propagation}
\label{sec:entity_learning}
\label{sec:local_global}

Temporal relations exhibit substantially different dynamics: some evolve slowly over long time spans (e.g., diplomatic alliances), while others are highly localized and depend primarily on recent events (e.g., daily conflict responses). \ModelName{} therefore injects temporal information via two complementary graphs: a \emph{global entity graph} over the full TKG for slow-changing relations, and a \emph{local entity graph} restricted to a window around the query timestamp for fast-changing ones. Given query-conditioned relation representations $\mathbf{r}_p$ from the relational backbone, \ModelName{} runs message passing on both graphs simultaneously.

\parabold{Global quadruple representation.}
The \emph{global entity graph} performs message passing over the entire temporal knowledge graph, capturing long-range temporal dependencies. For a temporal query $(s, p, ?, \tau_i)$, the initial entity representation is:
\begin{equation}
    \mathbf{e}_{v|s}^{0} = \mathbf{1}_{v=s} \cdot \mathbf{r}_p, \quad v \in G,
\end{equation}
where $\mathbf{r}_p$ is the query-conditioned relation representation from the relational backbone. At each message-passing layer, \ModelName{} injects temporal information into propagation:
\begin{align}
\label{eq:tmsgglobal}
\mathbf{e}_{v|s}^{l+1} &= \operatorname{AGG} \Big(
\operatorname{T\text{-}MSG}\!\big(\mathbf{e}_{w}^l,\, \mathbf{r}_q,\, \textstyle\sum g^{l+1}(\operatorname{TE}_{\tau_j})\big) \ \Big| \notag \\
&\qquad (e_w, q, v, \tau_j) \in G,\ w \in N_q(v),\ q \in R
\Big),
\end{align}
where $\operatorname{TE}_{\tau_j}$ is the temporal encoding of timestamp $\tau_j$ and $g^{l+1}(\cdot)$ is a learnable linear transformation per layer. Here $\operatorname{T\text{-}MSG}(\cdot)$ denotes the \emph{temporal message function}: it computes the message propagated along an edge by applying a temporal scoring function, such as T(NT)ComplEx~\cite{tcomplexlacroix2020tensor}, to the neighbor representation $\mathbf{e}_{w}^l$, the query-relation representation $\mathbf{r}_q$, and the injected temporal encoding. This is how time enters propagation: the temporal encoding modulates each message rather than being appended as a separate feature.

\parabold{Local quadruple representation.}
The \emph{local entity graph} restricts propagation to a temporal window centered around the query timestamp:
\begin{equation}
    G_{\mathrm{local}} = \{G_{i-k},\, \dots,\, G_i,\, \dots,\, G_{i+k}\},
\end{equation}
where $k$ controls the window size. The local representation follows the same temporal-aware propagation as Eq.~\ref{eq:tmsgglobal} but restricted to $G_{\mathrm{local}}$:
\begin{align}
\mathbf{e}_{v|s,\tau_i,\text{local}}^{0} &= \mathbf{1}_{v=s} \cdot \mathbf{r}_p, \quad v \in G_{\text{local}}, \\[2pt]
\label{eq:tmsglocal_update}
\mathbf{e}_{v|s,\tau_i,\text{local}}^{l+1} &= \operatorname{AGG} \Big(
\operatorname{T\text{-}MSG}\!\big(\mathbf{e}_{w,\text{local}}^l,\, \mathbf{r}_q,\, g^{l+1}(\operatorname{TE}_{\tau_j})\big) \Big| \notag \\
&\qquad (e_w, q, v, \tau_j) \in G_{\text{local}},\ w \in N_q(v),\ q \in R
\Big).
\end{align}

\parabold{Context fusion and scoring.}
The final entity representation combines local and global contexts:
\begin{equation}
    V_{s,p,o} = \alpha\, \mathbf{e}_{o|s,\tau_i,\text{local}} + (1-\alpha)\, \mathbf{e}_{o|s},
\end{equation}
where $\alpha$ balances short-term and long-range temporal information. The quadruple score is then computed as:
\begin{equation}
\label{eq:scorefunctionequation}
    S(s, p, o, \tau_i) = f_\theta\!\bigl(V_{s,p,o},\, \operatorname{TE}_{\tau_i}\bigr),
\end{equation}
where $f_\theta$ is a multilayer perceptron. By integrating transferable temporal encodings into both global and local message propagation, \ModelName{} captures both immediate temporal interactions and long-range temporal dependencies, improving cross-domain temporal transferability.

\subsection{Loss Function}
\label{sec:loss}

Given a temporal query $(s, p, ?, \tau)$, \ModelName{} predicts scores for candidate entities using the temporal-aware representations described above. We train the model using a binary cross-entropy objective with negative sampling. For each positive quadruple $(s, p, o, \tau)$, negative samples are generated by corrupting the head or tail entity. The training objective is:
\begin{equation}
    \mathcal{L} = -\sum_{(s,p,o,\tau)\in Q}
    \Bigl[y\log\sigma\!\bigl(S(s,p,o,\tau)\bigr)
    + (1-y)\log\!\bigl(1-\sigma\!\bigl(S(s,p,o,\tau)\bigr)\bigr)\Bigr],
\end{equation}
where $y \in \{0,1\}$ denotes the label of the quadruple and $\sigma(\cdot)$ is the sigmoid function. This objective encourages \ModelName{} to distinguish valid temporal facts from corrupted quadruples while learning transferable temporal and structural representations.

\noindent\textbf{Theoretical analysis.} We present a formal analysis of \ModelName{}'s ability to model temporal transferability, diverse temporal periodicities, and various temporal frequencies in Appendix~\ref{sec:theoretical_analysis}.

\section{Experimental Setup}
\label{sec:experiment_setup}
\paragraph{Datasets}
To evaluate cross-domain transfer, we conduct link prediction experiments across six widely-used TKG benchmark datasets\footnote{Code, pre-trained checkpoints, and datasets are available at \url{https://github.com/shouhulantian/FITTER}.} that differ substantially in domain, temporal granularity, and time span: ICEWS14~\cite{garcia-duran-etal-2018-learning}, ICEWS05-15~\cite{garcia-duran-etal-2018-learning}, GDELT~\cite{trivedi2017know}, ICEWS18~\cite{jin2020recurrent}, YAGO~\cite{jin2020recurrent}, and WIKI\cite{leblay2018deriving}. These datasets cover a diverse range of domains and temporal characteristics:

\begin{enumerate}
    \item \textbf{Inference Type:} ICEWS14, ICEWS05-15, and GDELT were designed for transductive (temporal interpolation) inference, while ICEWS18, YAGO and WIKI were developed for semi-inductive (temporal extrapolation) tasks.
    
    \item \textbf{Domain Coverage:} ICEWS14, ICEWS05-15, and ICEWS18 are subsets of the Integrated Crisis Early Warning System (ICEWS)~\cite{lautenschlager2015icews}, consisting of news event data. GDELT is a large-scale graph focused on capturing human behavior events. YAGO and WIKI are derived from the knowledge bases for commonsense facts.
    
    \item \textbf{Temporal Granularity:} ICEWS14, ICEWS05-15, and ICEWS18 use daily timestamps. GDELT provides fine-grained 15-minute intervals, while YAGO and WIKI use yearly granularity.
    
    \item \textbf{Time Span:} ICEWS14 and GDELT, ICEWS18 cover only 1 year, while ICEWS05-15 spans 11 years. Besides, YAGO and WIKI span a long range around 200 years.
\end{enumerate}

We perform \textbf{fully-inductive cross-domain evaluation}: the model is trained on one TKG and evaluated directly on a different TKG without retraining, ensuring no overlap in entities, relations, or timestamps between training and inference graphs. Table \ref{table:dataset} in the Appendix shows the dataset splits.

\paragraph{Experimental protocol} For each cross-domain scenario, the model is trained on $\Gtrain$ (the source TKG training split), with the best checkpoint selected on the source TKG validation split. The trained model is then applied in a fully-inductive manner to the target TKG $\Ginf$, whose facts are partitioned as $\Qinf = \Qmsg \cup \Qvalid \cup \Qtest$ following Section~\ref{sec:inference_settings}. Concretely, $\Qmsg$ corresponds to the target TKG training split and is used solely for inference-time message passing to build entity representations; no gradient updates are performed on it. $\Qvalid$ (target TKG validation split) is used to tune inference-time hyperparameters such as the local window size $k$ and fusion weight $\alpha$. $\Qtest$ (target TKG test split) yields the final reported results. This ensures strict separation between training supervision on the source graph and fully-inductive inference on the target graph, with no overlap in entities, relations, or timestamps across $\Gtrain$ and $\Ginf$.

\paragraph{Baselines}

We evaluate \ModelName{} under two regimes:

\noindent\textbf{(i) Fully-inductive cross-domain transfer.}
For cross-domain transfer, we compare against INGRAM~\cite{lee2023ingram} and ULTRA~\cite{galkintowards}, two strong fully-inductive reasoning models originally developed for static knowledge graphs. Both methods learn vocabulary-agnostic relational representations that generalize to unseen entities and relations, making them natural baselines for cross-domain evaluation. We retrain both models on temporal knowledge graphs under the same experimental protocol used for \ModelName{}. Transductive temporal extrapolation models such as RE-NET~\cite{jin2020recurrent}, CyGNet~\cite{zhu2021learning}, and TiRGN~\cite{li2022tirgn} are architecturally incompatible with this setting, as they depend on entity-indexed parameters undefined for unseen entities; we discuss this in detail in Appendix~\ref{app:incompatible_baselines}.

\noindent\textbf{(ii) Transductive in-domain evaluation.}
For standard transductive evaluation, we compare \ModelName{} against representative TKGE models, including TComplEx and TNTComplEx~\cite{tcomplexlacroix2020tensor}, as well as recent fully-inductive KG reasoning methods adapted to the temporal setting, including TRIX~\cite{zhangtrix}, ULTRA~\cite{galkintowards}, and GraphOracle~\cite{du2026graphoracle}. GraphOracle is included only in the transductive setting: we identified evaluation-time leakage in the original results and, after correction, the model became computationally impractical for cross-domain evaluation (see Appendix~\ref{app:graphoracle}).

\paragraph{Evaluation Metrics}
We adopt link prediction task to evaluate our proposed model. 
During the test step, we follow the procedure of \cite{xu2020tero} to generate candidate quadruples. From a test quadruple $(s,p,o, \tau)$, we replace $s$ with all $\bar{s} \in V$ and $o$ with all $\bar{o} \in V$ to get candidate answer quadruples $(s,p,\bar{o}, \tau)$ and $ (\bar{s},p,o, \tau)$ to queries $(s,p,?,\tau)$ and $(?, p,o,\tau)$.  
For each query, all candidate answer quadruples will be ranked by their scores using a time-aware filtering strategy \cite{goel2020diachronic}. We evaluate our models with three metrics: Mean Reciprocal Rank (MRR), the mean of the reciprocals of predicted ranks of correct quadruples, and Hits@(1/10), the percentage of ranks not higher than 1/10. For all experiments, the higher the better. 
We provide the hyperparameters and training details in Section \ref{sec:parameter} in the Appendix.

\begin{table*}[t!]
\centering
\caption{Fully-inductive link prediction results. Each block shows training on one dataset and testing on the remaining \textit{without fine-tuning}. Best results are in \textbf{bold}.}
\label{tbl:LinkPredictionResults}
\begin{minipage}{\textwidth}
\begin{adjustbox}{width=\textwidth}
\begin{tabular}{lccccccccccccccc||ccc}
\toprule
\multicolumn{19}{c}{\textbf{Trained on ICEWS14}} \\
\midrule
\multirow{2}{*}{Model} 
& \multicolumn{3}{c}{\textbf{To ICEWS05-15}} 
& \multicolumn{3}{c}{\textbf{To GDELT}} 
& \multicolumn{3}{c}{\textbf{To ICEWS18}} 
& \multicolumn{3}{c}{\textbf{To YAGO}} 
& \multicolumn{3}{c}{\textbf{To WIKI}} 
& \multicolumn{3}{c}{\textbf{Total Avg}}\\
& MRR & H@1 & H@10 & MRR & H@1 & H@10 & MRR & H@1 & H@10 & MRR & H@1 & H@10 & MRR & H@1 & H@10 & MRR & H@1 & H@10\\
\midrule
INGRAM      & 6.3  & 1.3  & 16.3 & 7.7  & 3.0  & 15.9 & 10.3 & 7.0  & 13.4 & 17.5 & 13.2 & 25.6 & 1.7 & 0.4 & 4.2 & 8.7 & 5.0 & 15.1 \\
ULTRA       & 30.1 & 21.4 & 52.3 & 17.0 & 10.0 & 30.0 & \underline{18.8} & \underline{9.5} & \underline{43.4} & 63.7 & 50.0 & 86.5 & \underline{42.4} & \underline{29.2} & \underline{67.6} & 34.4 & 24.0 & \underline{56.0}\\
TRIX        & \underline{32.7} & \underline{22.4} & \underline{53.3} & \underline{17.2} & \underline{10.4} & \underline{33.0} & 18.0 & 9.2 & 36.9 & \underline{72.6} & \underline{64.7} & \underline{87.2} & 41.4 & 29.1 & 65.5 & \underline{36.4} & \underline{27.2} & 55.2 \\
\midrule
\textbf{\ModelName{}} & \textbf{42.9} & \textbf{35.3} & \textbf{55.3} & \textbf{26.1} & \textbf{17.2} & \textbf{43.8} & \textbf{22.8} & \textbf{12.8} & \textbf{44.4} & \textbf{79.2} & \textbf{73.6} & \textbf{88.6} & \textbf{51.5} & \textbf{39.4} & \textbf{72.0} & \textbf{44.5} & \textbf{35.7} & \textbf{60.8}\\
\bottomrule
\end{tabular}
\end{adjustbox}
\end{minipage}


\centering
\begin{minipage}{\textwidth}
\begin{adjustbox}{width=\textwidth}
\begin{tabular}{lccccccccccccccc||ccc}
\toprule
\multicolumn{19}{c}{\textbf{Trained on ICEWS05-15}} \\
\midrule
\multirow{2}{*}{Model} 
& \multicolumn{3}{c}{\textbf{To ICEWS14}} 
& \multicolumn{3}{c}{\textbf{To GDELT}} 
& \multicolumn{3}{c}{\textbf{To ICEWS18}} 
& \multicolumn{3}{c}{\textbf{To YAGO}}
& \multicolumn{3}{c}{\textbf{To WIKI}}
& \multicolumn{3}{c}{\textbf{Total Avg}}\\
& MRR & H@1 & H@10 & MRR & H@1 & H@10 & MRR & H@1 & H@10 & MRR & H@1 & H@10 & MRR & H@1 & H@10 & MRR & H@1 & H@10 \\
\midrule
INGRAM      & 11.1 & 3.8  & 27.8 & 3.8  & 1.0  & 7.1  & 6.2  & 2.0  & 15.3 & 12.4 & 7.6  & 21.2 & 6.3 & 4.1 & 10.5 & 8.4 & 3.6 & 17.9 \\
ULTRA       & \underline{46.5} & \underline{34.3} & 70.4 & \underline{17.9} & \textbf{10.2} & \underline{32.4} & 14.5 & 8.3 & \textbf{32.1} & \underline{61.2} & 46.6 & \textbf{90.9} & 40.2 & 27.4 & 64.8 & \underline{36.5} & \underline{25.6} & \underline{58.1} \\
TRIX        & 45.5 & 33.3 & \textbf{71.0} & \underline{17.9} & \textbf{10.2} & 32.3 & \underline{14.6} & \underline{8.4} & \textbf{32.1} & \underline{61.2} & \underline{46.7} & 87.5 & \underline{43.3} & \underline{29.4} & \underline{65.8} & \underline{36.5} & \underline{25.6} & 57.7 \\
\midrule
\textbf{\ModelName{}} & \textbf{51.5} & \textbf{41.5} & 70.6 & \textbf{18.9} & \textbf{10.2} & \textbf{36.7} & \textbf{15.8} & \textbf{8.7} & 29.7 & \textbf{72.2} & \textbf{65.0} & 86.5 & \textbf{50.0} & \textbf{37.8} & \textbf{71.0} & \textbf{41.7} & \textbf{32.6} & \textbf{58.9} \\
\bottomrule
\end{tabular}
\end{adjustbox}
\end{minipage}



\centering
\begin{minipage}{\textwidth}
\begin{adjustbox}{width=\textwidth}
\begin{tabular}{lccccccccccccccc||ccc}
\toprule
\multicolumn{19}{c}{\textbf{Trained on GDELT}} \\
\midrule
\multirow{2}{*}{Model} 
& \multicolumn{3}{c}{\textbf{To ICEWS14}} 
& \multicolumn{3}{c}{\textbf{To ICEWS05-15}} 
& \multicolumn{3}{c}{\textbf{To ICEWS18}} 
& \multicolumn{3}{c}{\textbf{To YAGO}} 
& \multicolumn{3}{c}{\textbf{To WIKI}} 
& \multicolumn{3}{c}{\textbf{Total Avg}}\\
& MRR & H@1 & H@10 & MRR & H@1 & H@10 & MRR & H@1 & H@10 & MRR & H@1 & H@10 & MRR & H@1 & H@10 & MRR & H@1 & H@10 \\
\midrule
INGRAM      & 8.6  & 2.9  & 20.7 & 5.8  & 1.2  & 14.4 & 5.0  & 2.3  & 9.3  & 7.4  & 3.3  & 15.0 & 5.3 & 3.9 & 7.3 & 6.7 & 2.4 & 14.9 \\
ULTRA       & 33.6 & 24.5 & 50.7 & 38.3 & 26.8 & 60.5 & 14.4 & 8.6 & 31.8 & 48.5 & 40.1 & 63.2 & 36.5 & 27.2 & 61.2 & 34.3 & 25.4 & 53.5 \\
TRIX        & \underline{34.2} & \underline{25.2} & \underline{51.7} & \underline{39.3} & \underline{27.8} & \underline{61.0} & \underline{15.5} & \underline{9.3} & \underline{32.9} & \underline{51.3} & \underline{42.3} & \underline{65.5} & \underline{37.8} & \underline{29.2} & \underline{63.0} & \underline{35.6} & \underline{26.8} & \underline{54.8} \\
\midrule
\textbf{\ModelName{}} & \textbf{43.7} & \textbf{32.4} & \textbf{65.7} & \textbf{45.5} & \textbf{33.1} & \textbf{70.0} & \textbf{20.2} & \textbf{10.2} & \textbf{40.4} & \textbf{63.3} & \textbf{50.2} & \textbf{76.8} & \textbf{50.7} & \textbf{38.8} & \textbf{70.7} & \textbf{44.7} & \textbf{33.0} & \textbf{64.7} \\
\bottomrule
\end{tabular}
\end{adjustbox}
\end{minipage}
\end{table*}

\section{Experimental Result}
\label{sec:result}
Our central claim is that \ModelName{} enables fully-inductive cross-domain transfer across temporal knowledge graphs. To evaluate this, we investigate the following questions: \textbf{RQ1:} Does \ModelName{} achieve strong cross-domain transfer across TKGs with different domains, temporal granularities, and time spans? \textbf{RQ2:} Does the proposed temporal embedding effectively capture transferable temporal structural patterns? \textbf{RQ3:} Which components are essential for enabling effective cross-domain transfer?

\subsection{Cross-Domain Transfer Performance}
To answer \textbf{RQ1}, we evaluate \ModelName{} in fully-inductive cross-domain transfer: the model is trained on one TKG and applied directly to a structurally and temporally different TKG without retraining (Table~\ref{tbl:LinkPredictionResults}). \ModelName{} consistently and substantially outperforms all inductive baselines across all 15 transfer scenarios, confirming that our vocabulary-agnostic representations successfully transfer across domain boundaries. From the results, we highlight the following observations: 

%


\newtcolorbox[auto counter,crefname={Takeaway}{Takeaways}]%
  {takeawaybox}[2][]{
    boxsep=1pt,
    left=4pt, right=4pt, top=2pt, bottom=2pt,
    before skip=4pt, after skip=4pt,
    enhanced,
    sharp corners,
    #1
  }
 
\begin{takeawaybox}[label={takeaway:temporal_granularity}]
 
\textbf{Takeaway 1.}
   \ModelName{} generalizes well to datasets with varying \textbf{temporal granularities} and \textbf{spans}.
\end{takeawaybox}

The time span of the experimental datasets ranges from just 1 month (GDELT) to 189 years (YAGO), while their temporal granularities vary from 15 minutes (GDELT) to 1 year (YAGO). Despite being trained on one dataset and evaluated on another with significantly different temporal characteristics, \ModelName{} consistently achieves strong performance across all scenarios. This demonstrates the robustness and generalization ability of our sequential temporal embedding in handling diverse and unseen temporal information regardless of granularity and span changes. Furthermore, the performance gap between \ModelName{} and ULTRA is more pronounced when trained on ICEWS14 compared to ICEWS05-15, suggesting that the proposed sequential temporal embedding is particularly beneficial for smaller datasets.

\begin{takeawaybox}[label={takeaway:temporal_domain}]
 
\textbf{Takeaway 2.}
   \ModelName{} generalizes well across datasets from \textbf{different domains} and \textbf{densities}.
\end{takeawaybox}

The experimental datasets cover a wide range of domains, from encyclopedic knowledge in YAGO to diplomatic event data in ICEWS. \ModelName{} achieves impressive results even when trained on one domain and evaluated on another (see ``Trained on ICEWS14 to YAGO''), demonstrating the model’s ability to transfer across domains. This highlights that the learned quadruple representations from the relation encoder and quadruple encoder are capable of capturing domain-specific structural knowledge. Moreover, the datasets also vary in density: GDELT has more frequent events per timestamp, while ICEWS14 is relatively sparse. \ModelName{} consistently performs well in cross-dataset evaluation, demonstrating its robustness in event densities.

\begin{takeawaybox}[label={takeaway:task}]
 
\textbf{Takeaway 3.}
   \ModelName{} generalizes well to both \textbf{temporal knowledge graph interpolation} and \textbf{extrapolation} tasks.
\end{takeawaybox}

\ModelName{} achieves strong results when trained on the interpolation task setting and tested on the extrapolation settings (see ``Trained on ICEWS14 to ICEWS18''). We further observe that ULTRA retains an advantage over \ModelName{} on Hits@10 for YAGO when trained on ICEWS05-15 (90.9 vs.\ 86.5), even though \ModelName{} leads on MRR. Notably, over 90\% of test queries in YAGO follow a \textit{Strict Recurrency} pattern, where the same query appears in the immediately preceding snapshot \cite{gastinger2024history}. When historical recurrence is dominant, the local entity graph already places a large fraction of correct answers in the top-10, limiting the additional benefit of temporal embeddings for recall-oriented metrics such as Hits@10. 

However, when trained on GDELT and tested on YAGO, the performance gap becomes much more evident. Only 2.2\% of test queries in GDELT follow \textit{Strict Recurrency} pattern, compared to around 10\% in ICEWS14. This low \textit{Strict Recurrency} pattern distribution limits the learning on the local entity graph and transferability from GDELT to YAGO. Therefore, temporal embeddings in \ModelName{} become essential to transfer the \textit{Strict Recurrency} pattern from GDELT to YAGO. This highlights the importance of temporal embeddings for Temporal Transferability across TKGs by modeling relative temporal ordering. A similar performance trend is observed when testing on WIKI. However, the performance gap between ULTRA and \ModelName{} is larger, as WIKI contains a lower proportion of \textit{Strict Recurrency} queries\footnote{We evaluate transfer to temporal extrapolation datasets but do not \emph{pre-train} on them. Extrapolation training is unidirectional (past-only) and overemphasises the \textit{Strict Recurrency} pattern, which limits the learning of the bidirectional temporal structure needed for interpolation; empirically, extrapolation pre-training reduces transfer MRR to interpolation targets by roughly $20\%$. Interpolation datasets do not heavily depend on this pattern, so interpolation pre-training yields temporal representations that transfer well to both settings.}.

We further compare \ModelName{} against LLM-based temporal reasoning methods on extrapolation datasets in Table~\ref{tbl:Comparison_with_LLM}. ICL~\cite{lee2023temporal} and GenTKG~\cite{liao2024gentkg} are designed solely for temporal extrapolation; because scoring all candidate entities is computationally prohibitive for LLMs, they do not report MRR. Despite being a purely structure-driven model with no access to textual descriptions or language model priors, \ModelName{} achieves higher MRR and Hits@10 on both ICEWS18 and YAGO. GenTKG attains higher Hits@1 on ICEWS18, partly because it pre-filters candidates using entity-relation co-occurrence statistics, which narrows the prediction space. \ModelName{} does not rely on such filtering and generalises to both interpolation and extrapolation settings in a single unified framework.


\begin{table}[t]
    \centering
    \caption{
        Comparison with LLM-based temporal reasoning models on extrapolation datasets, trained on ICEWS14.
        ICL and GenTKG are evaluated only on ICEWS18 and YAGO, as they are designed solely for temporal extrapolation.
        LLM-based models cannot compute MRR because scoring all candidate entities is computationally prohibitive; their results are therefore fixed regardless of training source.
        Best results in \textbf{bold}.
    }
    \label{tbl:Comparison_with_LLM}
    \begin{adjustbox}{width=0.49\textwidth}
    \begin{tabular}{lcccccc}
        \toprule
        \multirow{2}{*}{\textbf{Model}}
        & \multicolumn{3}{c}{\textbf{ICEWS18}}
        & \multicolumn{3}{c}{\textbf{YAGO}} \\
        \cmidrule(lr){2-4} \cmidrule(lr){5-7}
        & MRR & H@1 & H@10 & MRR & H@1 & H@10 \\
        \midrule
        ICL~\cite{lee2023temporal}    & --            & 18.2          & 41.4          & --            & 72.6          & 84.6 \\
        GenTKG~\cite{liao2024gentkg}  & --            & \textbf{24.3} & 42.1          & --            & \textbf{79.2} & 84.3 \\
        \midrule
        \ModelName{}  & \textbf{22.8} & 12.8          & \textbf{44.4} & \textbf{79.2} & 73.6          & \textbf{88.6} \\
        \bottomrule
    \end{tabular}
    \end{adjustbox}
\end{table}

\subsection{Transductive Competitiveness}
Table~\ref{tbl:pretraining_results} compares \ModelName{} with TComplEx~\cite{tcomplexlacroix2020tensor}, TNTComplEx~\cite{tcomplexlacroix2020tensor}, TRIX~\cite{zhangtrix}, and ULTRA~\cite{galkintowards} in the standard supervised transductive setting. \ModelName{} outperforms both TRIX and ULTRA across all datasets on MRR, demonstrating that the temporal encoding in \ModelName{} provides a clear benefit even in the transductive setting. At the same time, \ModelName{} does not match the fully-trained transductive models TComplEx and TNTComplEx, which fit per-dataset embedding tables that scale with entity, relation, and timestamp vocabulary sizes. Despite this, \ModelName{} recovers 89--96\% of MRR and 97--99\% of Hits@10 on ICEWS14 and ICEWS05-15 using a single fixed vocabulary-agnostic model ($\approx$248--275K parameters). The gap is larger on GDELT (MRR recovery $\approx$71\%), where fine-grained 15-minute timestamps and dense event structure give dataset-specific embeddings more room to specialize. These results show that the cost of vocabulary-agnostic cross-domain generalization is modest on standard benchmarks, while unlocking the fully-inductive cross-domain transfer capability demonstrated in Table~\ref{tbl:LinkPredictionResults}.

\begin{table}[t]
    \centering
    \caption{
Transductive (in-domain) results. TComplEx and TNTComplEx use per-dataset embedding tables; GraphOracle, TRIX, ULTRA, and \ModelName{} use fixed vocabulary-agnostic models. Best results in \textbf{bold}; best among vocabulary-agnostic baselines \underline{underlined}. $\dagger$~GraphOracle results are post-correction after fixing two eval-time data leakage bugs; see Appendix~\ref{app:graphoracle}.
    }
    \label{tbl:pretraining_results}
    \resizebox{\columnwidth}{!}{
    \begin{tabular}{lccccccccc}
        \toprule
        \multirow{2}{*}{\textbf{Model}} &
         \multicolumn{3}{c}{\textbf{ICEWS14}} &
         \multicolumn{3}{c}{\textbf{ICEWS05-15}} & \multicolumn{3}{c}{\textbf{GDELT}} \\
        \cmidrule(lr){2-4} \cmidrule(lr){5-7} \cmidrule(lr){8-10}
        & MRR & H@10 & \#Param & MRR & H@10 & \#Param & MRR & H@10 & \#Param\\
        \midrule
        TComplEx   & \textbf{61.9} & 76.7          & 2,035,968 & 66.5          & 81.1          & 3,841,792 & \textbf{34.6} & \textbf{51.5} & 333,104 \\
        TNTComplEx & 60.7          & \textbf{77.2} & 2,624,856 & \textbf{66.6} & \textbf{81.7} & 3,101,800 & 34.1          & \textbf{51.5} & 354,300 \\
        \hline
        GraphOracle$^\dagger$ & 37.6 & 62.8       & 73,329    & 38.6          & 63.1          & 73,329    & --            & --            & --      \\
        TRIX       & 48.6          & 71.9          & 342,210   & \underline{45.6} & \underline{72.2} & 342,210 & \underline{21.0} & \underline{47.2} & 342,210 \\
        ULTRA      & \underline{49.5} & \underline{73.1} & 169,750 & 45.1       & 72.0          & 169,750   & 20.0          & 45.5          & 169,750 \\
        \hline
        \ModelName{} & 58.3        & 76.0          & 247,937   & 59.0          & 78.7          & 274,937   & 24.5          & 42.7          & 247,937 \\
        \textit{Gain vs.\ ULTRA} & \textit{+8.8} & \textit{+2.9} & -- & \textit{+13.9} & \textit{+6.7} & -- & \textit{+4.5} & \textit{$-$2.8} & -- \\
        \textit{Recovery vs.\ transductive baselines} & \textit{94.2\%} & \textit{98.4\%} & -- & \textit{88.6\%} & \textit{96.3\%} & -- & \textit{70.8\%} & \textit{82.9\%} & -- \\
        \bottomrule
    \end{tabular}}
\end{table}

\subsection{Temporal Analysis and Ablation Study}

\noindent\textbf{Temporal structural patterns (RQ2).}
To evaluate the ability of TE to capture temporal structural patterns, we compare \ModelName{} with ULTRA on two representative patterns using the subset from \cite{pan2024hge}: \textit{symmetric} quadruples, where $(s,p,o,\tau_1)$ implies $(o,p,s,\tau_2)$, and \textit{evolving} quadruples, where $(s,p,o,\tau_1)$ transitions to $(s,p’,o,\tau_2)$. As shown in Table~\ref{tbl:Pattern_performance}, \ModelName{} consistently outperforms ULTRA on both patterns, demonstrating that incorporating relative temporal signals enhances the model’s ability to learn and generalize temporal structural patterns.

\noindent\textbf{Ablation study (RQ3).}
Table~\ref{tbl:Ablation_Study} ablates the key components of \ModelName{}. GQR alone notably decreases Hits@1 but improves Hits@10, reflecting its role in capturing long-range context, while LQR alone reduces Hits@10 but improves Hits@1, reflecting its focus on local interactions; combining both is essential. Adding TE to LQR+GQR yields substantial gains: +8.1 MRR and +10.8 H@1 on transductive, and +4.0 MRR and +8.9 H@1 on cross-domain evaluation, confirming that transferable temporal encoding contributes significantly beyond structural propagation alone. Different TE integration strategies (TE-cat vs.\ TE-add) and T-MSG scoring variants show only minor differences, indicating the integration format is secondary.

\begin{table}[t]
    \centering
    \caption{Left: performance on temporal structural pattern subsets (trained on ICEWS14, tested on ICEWS05-15 subsets). Right: ablation study for key components (trained on ICEWS14, evaluated on ICEWS14 (transductive) and ICEWS05-15 (cross-domain)). T-MSG variants refer to different temporal message scoring functions.}
    \begin{subtable}[t]{0.44\linewidth}
        \centering
        \subcaption{Temporal structural patterns.}
        \label{tbl:Pattern_performance}
        \resizebox{\linewidth}{!}{%
        \begin{tabular}{lcccccc}
            \toprule
            \multirow{2}{*}{\textbf{Model}}
            & \multicolumn{3}{c}{\textbf{Symmetric}}
            & \multicolumn{3}{c}{\textbf{Evolve}} \\
            \cmidrule(lr){2-4} \cmidrule(lr){5-7}
            & MRR & H@1 & H@10 & MRR & H@1 & H@10 \\
            \midrule
            ULTRA             & 52.5          & 39.8          & 85.8          & 46.2          & 32.5          & 73.9 \\
            \ModelName{}+TE   & \textbf{67.0} & \textbf{53.7} & \textbf{90.7} & \textbf{59.8} & \textbf{46.7} & \textbf{84.3} \\
            \bottomrule
        \end{tabular}}
    \end{subtable}
    \hfill
    \begin{subtable}[t]{0.53\linewidth}
        \centering
        \subcaption{Ablation study.}
        \label{tbl:Ablation_Study}
        \resizebox{\linewidth}{!}{%
        \begin{tabular}{lcccccc}
            \toprule
            \multirow{2}{*}{\textbf{Strategy}}
            & \multicolumn{3}{c}{\textbf{Transductive}}
            & \multicolumn{3}{c}{\textbf{Fully-inductive}} \\
            \cmidrule(lr){2-4} \cmidrule(lr){5-7}
            & MRR & H@1 & H@10 & MRR & H@1 & H@10 \\
            \midrule
            \ModelName{}   & \textbf{58.3} & \textbf{48.9} & \textbf{76.0} & \textbf{42.9} & \textbf{35.3} & 55.3 \\
            \midrule
            GQR only           & 49.2 & 37.3 & 72.6 & 37.9 & 25.6 & \textbf{62.6} \\
            LQR only           & 46.5 & 39.0 & 60.0 & 40.0 & 32.5 & 52.6 \\
            LQR+GQR w/o TE     & 50.2 & 38.1 & 74.3 & 38.9 & 26.4 & 62.1 \\
            \midrule
            w/o TE-cat         & 56.3 & 47.0 & 74.8 & 41.5 & 33.4 & 53.0 \\
            with TE-add        & 57.3 & 47.0 & 75.8 & 42.7 & 34.0 & 55.0 \\
            \midrule
            T-MSG (TTransE)    & 56.3 & 45.6 & 75.8 & 40.9 & 33.3 & 54.3 \\
            T-MSG (TNTComplEx) & 58.0 & 48.0 & 75.8 & 41.3 & 34.3 & 55.3 \\
            \bottomrule
        \end{tabular}}
    \end{subtable}
\end{table}

\section{Conclusion}
We present \ModelName{}, the first temporal KG embedding model for \emph{cross-domain transfer}, requiring no dataset-specific entity, relation, or timestamp embeddings. \ModelName{} uses sinusoidal positional encodings for granularity-agnostic temporal representation and relation interaction graphs for vocabulary-agnostic structural transfer. Across 15 cross-dataset transfer scenarios spanning six TKGs, \ModelName{} consistently outperforms inductive baselines while remaining competitive with transductive models. Future directions include multi-source pre-training and scaling to more diverse TKG collections.

\section*{Acknowledgements}
We gratefully acknowledge computing time on the HoreKa supercomputer at the National High-Performance Computing Center at KIT (NHR@KIT), jointly funded by the German Federal Ministry of Education and Research (BMBF), the Ministry of Science, Research, and the Arts of Baden-W\"urttemberg, and the German Research Foundation (DFG). Jiaxin Pan and Osama Mohammed acknowledge funding from the EU Chips Joint Undertaking (GA 101140087, SMARTY) and from the BMBF sub-project 16MEE0444. Daniel Hern\'andez is funded by the German Research Foundation (DFG) -- SFB 1574 -- 471687386.

\section*{Declaration of Use of Generative AI}
Claude (Anthropic) was utilized to assist with editing and refining sections of this work, including text. The authors remain fully responsible for all scientific content, experimental results, and conclusions presented in this paper. All writing, research design, implementation, and analysis were carried out by the authors; generative AI was used solely as an editing aid and did not contribute to the research itself.

\bibliographystyle{splncs04}
\bibliography{custom}

@article{vrandevcic2014wikidata,
  title={Wikidata: A Free Collaborative Knowledgebase},
  author={Vrande{\v{c}}i{\'c}, Denny and Kr{\"o}tzsch, Markus},
  journal={Communications of the ACM},
  volume={57},
  number={10},
  pages={78--85},
  year={2014},
  publisher={ACM}
}

@article{lehmann2015dbpedia,
  title={Dbpedia--a large-scale, multilingual knowledge base extracted from wikipedia},
  author={Lehmann, Jens and Isele, Robert and Jakob, Max and Jentzsch, Anja and Kontokostas, Dimitris and Mendes, Pablo N and Hellmann, Sebastian and Morsey, Mohamed and Van Kleef, Patrick and Auer, S{\"o}ren and others},
  journal={Semantic web},
  volume={6},
  number={2},
  pages={167--195},
  year={2015},
  publisher={SAGE Publications Sage UK: London, England}
}

@article{hoffart2013yago2,
  title={YAGO2: A spatially and temporally enhanced knowledge base from Wikipedia},
  author={Hoffart, Johannes and Suchanek, Fabian M and Berberich, Klaus and Weikum, Gerhard},
  journal={Artificial intelligence},
  volume={194},
  pages={28--61},
  year={2013},
  publisher={Elsevier}
}

@inproceedings{du2026graphoracle,
  title={GraphOracle: Efficient Fully-Inductive Knowledge Graph Reasoning via Relation-Dependency Graphs},
  author={Du, Enjun and Liu, Siyi and Zhang, Yongqi},
  booktitle={Proceedings of the AAAI Conference on Artificial Intelligence},
  volume={40},
  number={23},
  pages={19055--19063},
  year={2026}
}

@misc{cai2024surveytemporalknowledgegraph,
      title={A Survey on Temporal Knowledge Graph: Representation Learning and Applications}, 
      author={Li Cai and Xin Mao and Yuhao Zhou and Zhaoguang Long and Changxu Wu and Man Lan},
      year={2024},
      eprint={2403.04782},
      archivePrefix={arXiv},
      primaryClass={cs.CL},
      url={https://arxiv.org/abs/2403.04782}, 
}

@article{su2024roformer,
  title={Roformer: Enhanced transformer with rotary position embedding},
  author={Su, Jianlin and Ahmed, Murtadha and Lu, Yu and Pan, Shengfeng and Bo, Wen and Liu, Yunfeng},
  journal={Neurocomputing},
  volume={568},
  pages={127063},
  year={2024},
  publisher={Elsevier}
}

@inproceedings{gastinger2024history,
  title={History Repeats Itself: A Baseline for Temporal Knowledge Graph Forecasting},
  author={Gastinger, Julia and Meilicke, Christian and Errica, Federico and Sztyler, Timo and Sch{\"u}lke, Anett and Stuckenschmidt, Heiner},
  booktitle={IJCAI},
  year={2024}
}

@inproceedings{
tcomplexlacroix2020tensor,
title={Tensor Decompositions for Temporal Knowledge Base Completion},
author={Timothée Lacroix and Guillaume Obozinski and Nicolas Usunier},
booktitle={International Conference on Learning Representations},
year={2020},
url={https://openreview.net/forum?id=rke2P1BFwS}
}

@inproceedings{jin2020recurrent,
  title={Recurrent Event Network: Autoregressive Structure Inferenceover Temporal Knowledge Graphs},
  author={Jin, Woojeong and Qu, Meng and Jin, Xisen and Ren, Xiang},
  booktitle={Proceedings of the 2020 Conference on Empirical Methods in Natural Language Processing (EMNLP)},
  pages={6669--6683},
  year={2020}
}

@inproceedings{tltcomplexzhang2022along,
  title={Along the Time: Timeline-traced Embedding for Temporal Knowledge Graph Completion},
  author={Zhang, Fuwei and Zhang, Zhao and Ao, Xiang and Zhuang, Fuzhen and Xu, Yongjun and He, Qing},
  booktitle={Proceedings of the 31st ACM International Conference on Information \& Knowledge Management},
  pages={2529--2538},
  year={2022}
}

@inproceedings{pan2024hge,
  title={HGE: embedding temporal knowledge graphs in a product space of heterogeneous geometric subspaces},
  author={Pan, Jiaxin and Nayyeri, Mojtaba and Li, Yinan and Staab, Steffen},
  booktitle={Proceedings of the AAAI Conference on Artificial Intelligence},
  volume={38},

  pages={8913--8920},
  year={2024}
}

@inproceedings{trivedi2017know,
  title={Know-evolve: Deep temporal reasoning for dynamic knowledge graphs},
  author={Trivedi, Rakshit and Dai, Hanjun and Wang, Yichen and Song, Le},
  booktitle={international conference on machine learning},
  pages={3462--3471},
  year={2017},
  organization={PMLR}
}

@inproceedings{xu2020tero,
  title={TeRo: A Time-aware Knowledge Graph Embedding via Temporal Rotation},
  author={Xu, Chengjin and Nayyeri, Mojtaba and Alkhoury, Fouad and Yazdi, Hamed Shariat and Lehmann, Jens},
  booktitle={Proceedings of the 28th International Conference on Computational Linguistics},
  pages={1583--1593},
  year={2020}
}

@inproceedings{goel2020diachronic,
  title={Diachronic embedding for temporal knowledge graph completion},
  author={Goel, Rishab and Kazemi, Seyed Mehran and Brubaker, Marcus and Poupart, Pascal},
  booktitle={Proceedings of the AAAI Conference on Artificial Intelligence},
  volume={34},

  pages={3988--3995},
  year={2020}
}

@inproceedings{leblay2018deriving,
  title={Deriving validity time in knowledge graph},
  author={Leblay, Julien and Chekol, Melisachew Wudage},
  booktitle={Companion proceedings of the the web conference 2018},
  pages={1771--1776},
  year={2018}
}

@article{lautenschlager2015icews,
  title={Icews event aggregations},
  author={Lautenschlager, Jennifer and Shellman, Steve and Ward, Michael},
  journal={Harvard Dataverse},
  volume={3},
  number={595},
  pages={28},
  year={2015}
}

@inproceedings{garcia-duran-etal-2018-learning,
  title={Learning sequence encoders for temporal knowledge graph completion},
  author={Garc{\'\i}a-Dur{\'a}n, Alberto and Duman{\v{c}}i{\'c}, Sebastijan and Niepert, Mathias},
  booktitle={Proceedings of the 2018 conference on empirical methods in natural language processing},
  pages={4816--4821},
  year={2018}
}

@inproceedings{galkintowards,
  title={Towards Foundation Models for Knowledge Graph Reasoning},
  author={Galkin, Mikhail and Yuan, Xinyu and Mostafa, Hesham and Tang, Jian and Zhu, Zhaocheng},
  booktitle={The Twelfth International Conference on Learning Representations},
 year={2024}
}

@inproceedings{liang2023learn,
  title={Learn from relational correlations and periodic events for temporal knowledge graph reasoning},
  author={Liang, Ke and Meng, Lingyuan and Liu, Meng and Liu, Yue and Tu, Wenxuan and Wang, Siwei and Zhou, Sihang and Liu, Xinwang},
  booktitle={Proceedings of the 46th international ACM SIGIR conference on research and development in information retrieval},
  pages={1559--1568},
  year={2023}
}

@inproceedings{li2022complex,
  title={Complex Evolutional Pattern Learning for Temporal Knowledge Graph Reasoning},
  author={Li, Zixuan and Guan, Saiping and Jin, Xiaolong and Peng, Weihua and Lyu, Yajuan and Zhu, Yong and Bai, Long and Li, Wei and Guo, Jiafeng and Cheng, Xueqi},
  booktitle={Proceedings of the 60th Annual Meeting of the Association for Computational Linguistics (Volume 2: Short Papers)},
  pages={290--296},
  year={2022}
}

@inproceedings{sun2021timetraveler,
  title={TimeTraveler: Reinforcement Learning for Temporal Knowledge Graph Forecasting},
  author={Sun, Haohai and Zhong, Jialun and Ma, Yunpu and Han, Zhen and He, Kun},
  booktitle={Proceedings of the 2021 Conference on Empirical Methods in Natural Language Processing},
  pages={8306--8319},
  year={2021}
}

@inproceedings{han2020explainable,
  title={Explainable subgraph reasoning for forecasting on temporal knowledge graphs},
  author={Han, Zhen and Chen, Peng and Ma, Yunpu and Tresp, Volker},
  booktitle={International conference on learning representations},
  year={2020}
}

@inproceedings{liu2022tlogic,
  title={Tlogic: Temporal logical rules for explainable link forecasting on temporal knowledge graphs},
  author={Liu, Yushan and Ma, Yunpu and Hildebrandt, Marcel and Joblin, Mitchell and Tresp, Volker},
  booktitle={Proceedings of the AAAI conference on artificial intelligence},
  volume={36},

  pages={4120--4127},
  year={2022}
}

@inproceedings{li2022tirgn,
  title={TiRGN: Time-Guided Recurrent Graph Network with Local-Global Historical Patterns for Temporal Knowledge Graph Reasoning.},
  author={Li, Yujia and Sun, Shiliang and Zhao, Jing},
  booktitle={IJCAI},
  pages={2152--2158},
  year={2022}
}

@inproceedings{zhu2021learning,
  title={Learning from History: Modeling Temporal Knowledge Graphs with Sequential Copy-Generation Networks},
  author={Zhu, Cunchao and Chen, Muhao and Fan, Changjun and Cheng, Guangquan and Zhang, Yan},
  booktitle={Proceedings of the AAAI Conference on Artificial Intelligence},
  volume={35},
  number={5},
  pages={4732--4740},
  year={2021}
}

@article{zhu2021neural,
  title={Neural bellman-ford networks: A general graph neural network framework for link prediction},
  author={Zhu, Zhaocheng and Zhang, Zuobai and Xhonneux, Louis-Pascal and Tang, Jian},
  journal={Advances in neural information processing systems},
  volume={34},
  pages={29476--29490},
  year={2021}
}

@inproceedings{teru2020inductive,
  title={Inductive relation prediction by subgraph reasoning},
  author={Teru, Komal and Denis, Etienne and Hamilton, Will},
  booktitle={International conference on machine learning},
  pages={9448--9457},
  year={2020},
  organization={PMLR}
}

@article{liu2021indigo,
  title={Indigo: Gnn-based inductive knowledge graph completion using pair-wise encoding},
  author={Liu, Shuwen and Grau, Bernardo and Horrocks, Ian and Kostylev, Egor},
  journal={Advances in Neural Information Processing Systems},
  volume={34},
  pages={2034--2045},
  year={2021}
}

@inproceedings{chen2022meta,
  title={Meta-knowledge transfer for inductive knowledge graph embedding},
  author={Chen, Mingyang and Zhang, Wen and Zhu, Yushan and Zhou, Hongting and Yuan, Zonggang and Xu, Changliang and Chen, Huajun},
  booktitle={Proceedings of the 45th international ACM SIGIR conference on research and development in information retrieval},
  pages={927--937},
  year={2022}
}

@inproceedings{lee2023ingram,
  title={InGram: Inductive knowledge graph embedding via relation graphs},
  author={Lee, Jaejun and Chung, Chanyoung and Whang, Joyce Jiyoung},
  booktitle={International Conference on Machine Learning},
  pages={18796--18809},
  year={2023},
  organization={PMLR}
}

@inproceedings{lee2023temporal,
  title={Temporal knowledge graph forecasting without knowledge using in-context learning},
  author={Lee, Dong-Ho and Ahrabian, Kian and Jin, Woojeong and Morstatter, Fred and Pujara, Jay},
  booktitle={Proceedings of the 2023 conference on empirical methods in natural language processing},
  pages={544--557},
  year={2023}
}

@inproceedings{liao2024gentkg,
  title={GenTKG: Generative Forecasting on Temporal Knowledge Graph with Large Language Models},
  author={Liao, Ruotong and Jia, Xu and Li, Yangzhe and Ma, Yunpu and Tresp, Volker},
  booktitle={NAACL-HLT (Findings)},
  year={2024}
}

@inproceedings{ding2024zrllm,
  title={zrLLM: Zero-Shot Relational Learning on Temporal Knowledge Graphs with Large Language Models},
  author={Ding, Zifeng and Cai, Heling and Wu, Jingpei and Ma, Yunpu and Liao, Ruotong and Xiong, Bo and Tresp, Volker},
  booktitle={Proceedings of the 2024 Conference of the North American Chapter of the Association for Computational Linguistics: Human Language Technologies (Volume 1: Long Papers)},
  pages={1877--1895},
  year={2024}
}

@inproceedings{zhangtrix,
  title={TRIX: A More Expressive Model for Zero-shot Domain Transfer in Knowledge Graphs},
  author={Zhang, Yucheng and Bevilacqua, Beatrice and Galkin, Mikhail and Ribeiro, Bruno},
  booktitle={The Third Learning on Graphs Conference},
  year={2025}
}

@inproceedings{jia2021complex,
  title={Complex temporal question answering on knowledge graphs},
  author={Jia, Zhen and Pramanik, Soumajit and Saha Roy, Rishiraj and Weikum, Gerhard},
  booktitle={Proceedings of the 30th ACM international conference on information \& knowledge management},
  pages={792--802},
  year={2021}
}

@article{luo2024chain,
  title={Chain of history: Learning and forecasting with llms for temporal knowledge graph completion},
  author={Luo, Ruilin and Gu, Tianle and Li, Haoling and Li, Junzhe and Lin, Zicheng and Li, Jiayi and Yang, Yujiu},
  journal={arXiv preprint arXiv:2401.06072},
  year={2024}
}

\clearpage
\appendix
\renewcommand{\thesection}{\Alph{section}}
\setcounter{section}{0}

\section{Parameter Count and Complexity}
From Table~\ref{table:time and space}, we observe that the number of parameters in \ModelName{} is agnostic to the dataset size. Since \ModelName{} does not initialize embeddings based on $|V|$, $|R|$, or $|\Tau|$, but instead relies on a fixed number of relation interactions $|H|$, its parameter count is only related to the embedding dimension $d$. This property makes \ModelName{} especially suitable for transfer learning scenarios.

The time complexity of \ModelName{} is primarily determined by the quadruple encoder, as the number of relations $|R|$ is significantly smaller than the number of entities $|V|$, allowing us to omit the complexity contribution of the relation encoder. Furthermore, since the local entity graph is much smaller than the global entity graph, the computational upper bound is dominated by the global quadruple representation. Utilizing NBFNet~\cite{zhu2021neural} as the encoder, the time complexity for a single layer is $O(|Q|d + |V|d^2)$. For $A$ layers, the total time complexity becomes $O(A(|Q|d + |V|d^2))$.

The memory complexity of \ModelName{} is also linear in the number of edges, expressed as $O(A|Q|d)$, as the quadruple encoder maintains and updates representations for each edge in the temporal knowledge graph.

\paragraph{Inference-time cost.} We report wall-clock inference cost on the largest benchmark. Evaluating the full GDELT test split ($341{,}961$ quadruples) takes approximately $21$ hours on a single NVIDIA A100. The bottleneck is the NBFNet-style conditional message passing: it materialises and updates a representation for every edge and loads the entire inference graph into GPU memory, so the cost grows with graph size and density. GDELT is the worst case among our benchmarks because its $15$-minute granularity produces a very dense event graph. Importantly, this cost is a property of inference on large \emph{target} graphs and not a training requirement: \ModelName{} never needs to be trained on a large graph in order to transfer to one. A model trained on the small ICEWS14 graph ($1.8$\,h/epoch) already achieves $26.1$ MRR on GDELT without any retraining, so multi-source pre-training can rely entirely on multiple small and diverse TKGs. Reducing inference cost on dense target graphs is a natural direction for future work; promising strategies include query-conditioned subgraph extraction and attention-based edge selection that prune the message-passing graph before propagation.

\begin{table*}[h]
\centering
    \caption{
    Parameter number and average training time for \ModelName{}.
    }
    \label{table:time and space}
{ 
\begin{tabular}{lccccc}
    \toprule
 Model & Datasets & $d$ & Parameter number\ \ & Batch Size &\ \ Average epoch time(h) \cr 
\midrule
 \multirow{3}{*}{\ModelName{}}  
 & ICEWS14&64& 247,937 & 16&1.8  \cr
 &ICEWS05-15& 64 &247,937& 2 &21.5 \cr
 &GDELT & 64 & 247,937 & 1 & 123.3 \cr
\bottomrule
\end{tabular}
}
\end{table*}

\begin{table*}
\centering
    \caption{
    Detailed statistics of datasets.
    }
    \label{table:dataset}
\begin{tabular}{lcccccc}
  \toprule
 Dataset &ICEWS14&ICEWS05-15&GDELT&ICEWS18 & YAGO &WIKI\cr
  \midrule
 Entities  &7,128 &10,488  &500 &23,033 & 10,623&12,554 \cr
Relations & 230 & 251  & 20& 256 & 10 &24 \cr
Times & 365 & 4017 &366 & 303 &189 &232\cr
Train& 72,826 &386,962 & 2,735,685& 373,018 & 161,540 &539,286\cr
Validation & 8,941& 46,275 & 341,961& 45,995& 19,523& 67,538\cr 
Test & 8,963 & 46,092 & 341,961 & 49,995 & 20,026 &63,110\cr 
    Granularity & Daily & Daily & 15 minutes & Daily & Yearly&Yearly  \cr
        \bottomrule
\end{tabular}
\end{table*}

\section{GraphOracle: Leakage Bugs and Exclusion Rationale}
\label{app:graphoracle}

We attempted to include GraphOracle~\cite{du2026graphoracle} as a baseline but discovered two independent eval-time data leakage bugs in its data loader (\texttt{load\_data.py}) that render all published results unreliable. We report the corrected numbers in Table~\ref{tbl:pretraining_results} and exclude GraphOracle from cross-domain experiments for the additional reason of computational infeasibility.

\paragraph{Bug 1 — Relation-dependency graph leak.}
The relation edges fed to the \texttt{RelationGraphEncoder} at evaluation time were constructed from the validation/test split rather than the training split, allowing relation co-occurrence patterns from the evaluation data to influence per-relation embeddings.

\paragraph{Bug 2 — Knowledge graph substrate leak.}
The query-conditioned subgraph used for BFS-style message passing was similarly built from the evaluation split. Because the data loader adds both $(h, r, t)$ and its inverse $(t, r_\text{inv}, h)$ for every quadruple, the literal answer edge $(h, r, \text{answer})$ of every evaluation query is guaranteed to appear in the message-passing graph. A per-query walk-through confirmed that 100\% of validation queries on ICEWS14 (17,649/17,649) have the exact answer edge present in the leaky graph.

\paragraph{Empirical impact.}
Two controlled runs on ICEWS14 (hidden\_dim=80, n\_layer=3) quantify the inflation:

\begin{center}
\begin{tabular}{lcc}
\toprule
Run & Test MRR & Test H@10 \\
\midrule
Leaky (original) & 0.671 & 0.915 \\
Fully fixed      & 0.376 & 0.628 \\
\bottomrule
\end{tabular}
\end{center}

The combined inflation is $+0.295$ absolute MRR (78\% relative) on ICEWS14.

\paragraph{Computational infeasibility.}
After fixing the leakage bugs, GraphOracle's subgraph extraction step becomes a practical bottleneck: it expands the neighbourhood of all query entities layer by layer, and the number of edges grows rapidly with graph density. On GDELT, which has far more events per timestamp than ICEWS14, validation alone was projected to exceed 29 hours on 4$\times$A100 GPUs, making both training and cross-domain transfer infeasible. We therefore exclude GraphOracle from cross-domain experiments and report only its corrected transductive results in Table~\ref{tbl:pretraining_results}.

\section{Incompatibility of Transductive Temporal Extrapolation Models}
\label{app:incompatible_baselines}

A natural question is why we do not compare \ModelName{} against well-known temporal knowledge graph reasoning models such as RE-NET~\cite{jin2020recurrent}, CyGNet~\cite{zhu2021learning}, and TiRGN~\cite{li2022tirgn}. These are \emph{transductive temporal extrapolation} models: they are trained and evaluated on a single graph with a fixed entity and relation vocabulary, and they forecast future facts within that same vocabulary. They are therefore architecturally incompatible with our fully-inductive cross-domain setting, in which the target inference graph introduces entirely unseen entities, relations, and timestamps.

The root cause is that each of these models relies on \emph{entity-indexed} parameters that are undefined for entities not seen during training. We verified this by inspecting their reference implementations:

\begin{itemize}
    \item \textbf{RE-NET}~\cite{jin2020recurrent} maintains per-entity recurrent states (\texttt{ent\_embeds}) that are updated autoregressively over historical snapshots. Each state is indexed by a fixed entity id.
    \item \textbf{CyGNet}~\cite{zhu2021learning} builds a per-entity historical co-occurrence matrix that drives its copy-generation mechanism. The matrix is defined over the training entity vocabulary and cannot score entities absent from it.
    \item \textbf{TiRGN}~\cite{li2022tirgn} evolves per-entity dynamic embeddings through a GCN--GRU module (\texttt{dynamic\_emb}), again keyed by fixed entity ids.
\end{itemize}

Because the target graph presents a new entity vocabulary, none of these entity-indexed parameters can be instantiated at inference time. Adapting the models to accept unseen entities---for example, by replacing their entity-indexed states with a structural, vocabulary-agnostic encoder---would remove the mechanisms that define them and yield fundamentally different architectures. We therefore restrict our cross-domain baselines to fully-inductive, vocabulary-agnostic models (INGRAM~\cite{lee2023ingram} and ULTRA~\cite{galkintowards}), which are the only prior methods directly applicable to this setting.

\section{Limitations}
\label{sec:limitation}
Although \ModelName{} demonstrates strong fully-inductive inference performance in cross-dataset scenarios, several limitations remain. First, training on larger graphs does not always lead to improved performance. We hypothesize that variations in temporal patterns across datasets of different sizes may influence results. Second, \ModelName{} can be computationally intensive for dense temporal knowledge graphs, as its space complexity scales with the number of events. Third, \ModelName{} is unable to perform relation prediction and time prediction. We aim to explore more efficient architectures and adaptive training strategies to address these challenges in future work.

\section{Relative Representation Transfer}
\label{sec:relative_representation}
 Subgraph (a) in Figure \ref{fig:relative_representation} shows the transfer of relative entity representations. In the left bottom KG, $(\text{South Korea} \xrightarrow{\text { meet }} \text{North Korea} \xrightarrow{\text { negotiate }} \text{US}) \wedge (\text{South Korea} \xrightarrow{\text { dialogue with }} \text{China} \xrightarrow{\text { negotiate }} \text{US})$ implies that $\text{South Korea} \xrightarrow{\text { negotiate }} \text{US}$. Given entity $v_1$, $v_2$, $v_3$, $v_4$, we can learn a structure that $(v_1 \xrightarrow{\text { meet }} v_2 \xrightarrow{\text { negotiate }} v_3) \wedge (v_1 \xrightarrow{\text { dialogue with }} v_4 \xrightarrow{\text { negotiate }} v_3)$ implies that $v_1 \xrightarrow{\text { negotiate }} v_3$ during training. In the right part, we substitute $v_1$, $v_2$, $v_3$, $v_4$ with Germany, India, UK and EU respectively, then we can infer that $\text{Germany} \xrightarrow{\text { negotiate }} \text{UK}$.
 
Subgraph (b) depicts the transfer of relative relation representations. In the left part, the upper relation graph depicts the interaction of relations in the bottom KG: $(\text{dialogue with} \xrightarrow{\text { h2h }} \text{meet} \xrightarrow{\text { t2h }} \text{negotiate}) \wedge (\text{dialogue with} \xrightarrow{\text { t2h }} \text{negotiate})$. Given relation $r_1$, $r_2$, $r_3$, we learn a relational structure: $(r_1 \xrightarrow{\text { h2h }} r_2 \xrightarrow{\text { t2h }} r_3) \wedge (r_1 \xrightarrow{\text { t2h }} r_3)$ during training. In the right part, we substitute $r_1$, $r_2$, $r_3$ with \textit{marries}, \textit{father of} and \textit{lives in} respectively, then we can the relational structure for the KG in the right part.

Subgraph (c) depicts the transfer of relative temporal ordering. In the left part, the upper relation graph depicts the relative temporal ordering of quadruples in the bottom TKG: $(\text{dialogue with} \xrightarrow{\Delta \text{=6} } \text{meet} \xrightarrow{\Delta \text{=4}} \text{negotiate}) \wedge (\text{dialogue with} \xrightarrow{\Delta \text{=5}} \text{negotiate})$. Given relation $r_1$, $r_2$, $r_3$, we learn a relational structure: $(r_1 \xrightarrow{\Delta \text{=6}} r_2 \xrightarrow{\Delta \text{=4}} r_3) \wedge (r_1 \xrightarrow{\Delta \text{=5}} r_3)$ during training. In the right part, we substitute $r_1$, $r_2$, $r_3$ with \textit{marries}, \textit{father of} and \textit{lives in} respectively, then we can the temporal transition information for the TKG in the right part.

\begin{figure*}[htbp]
    \centering
    \begin{subfigure}[b]{0.32\textwidth}
        \includegraphics[width=\linewidth]{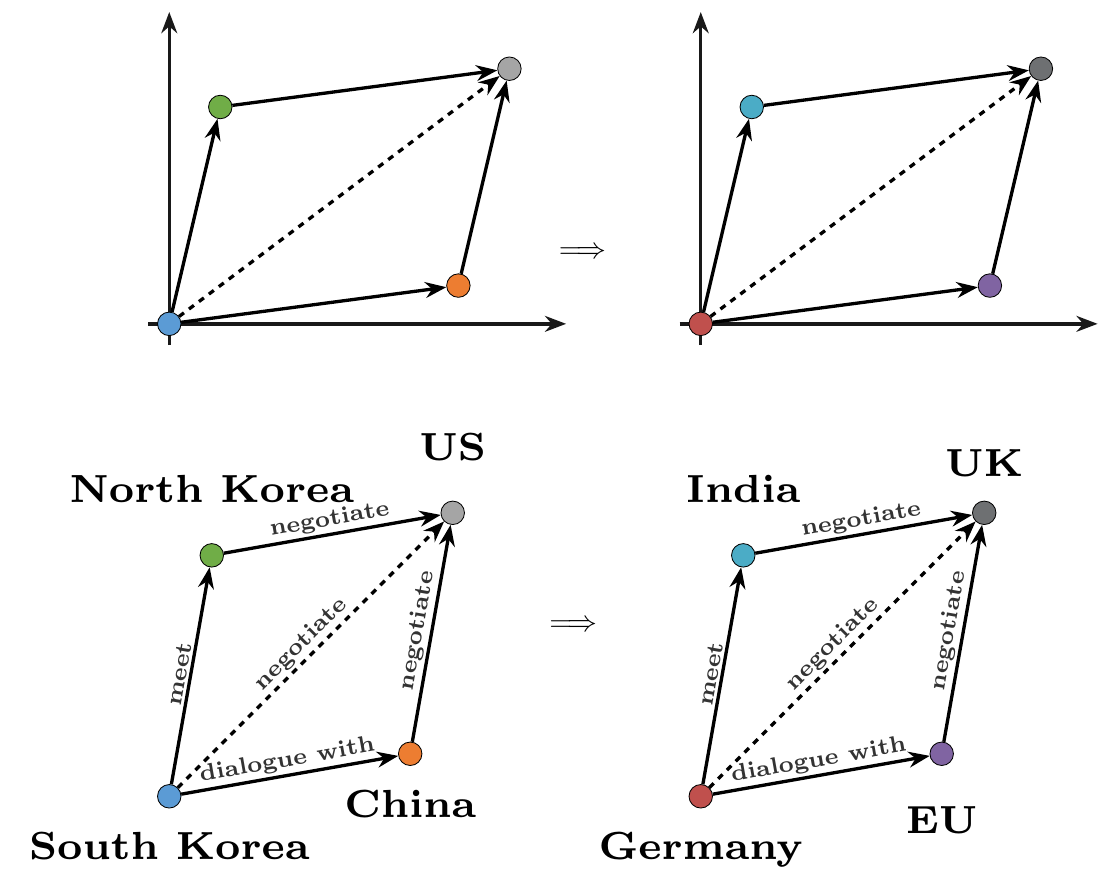}
        \caption{}
    \end{subfigure}
    \hfill
    \begin{subfigure}[b]{0.32\textwidth}
        \includegraphics[width=\linewidth]{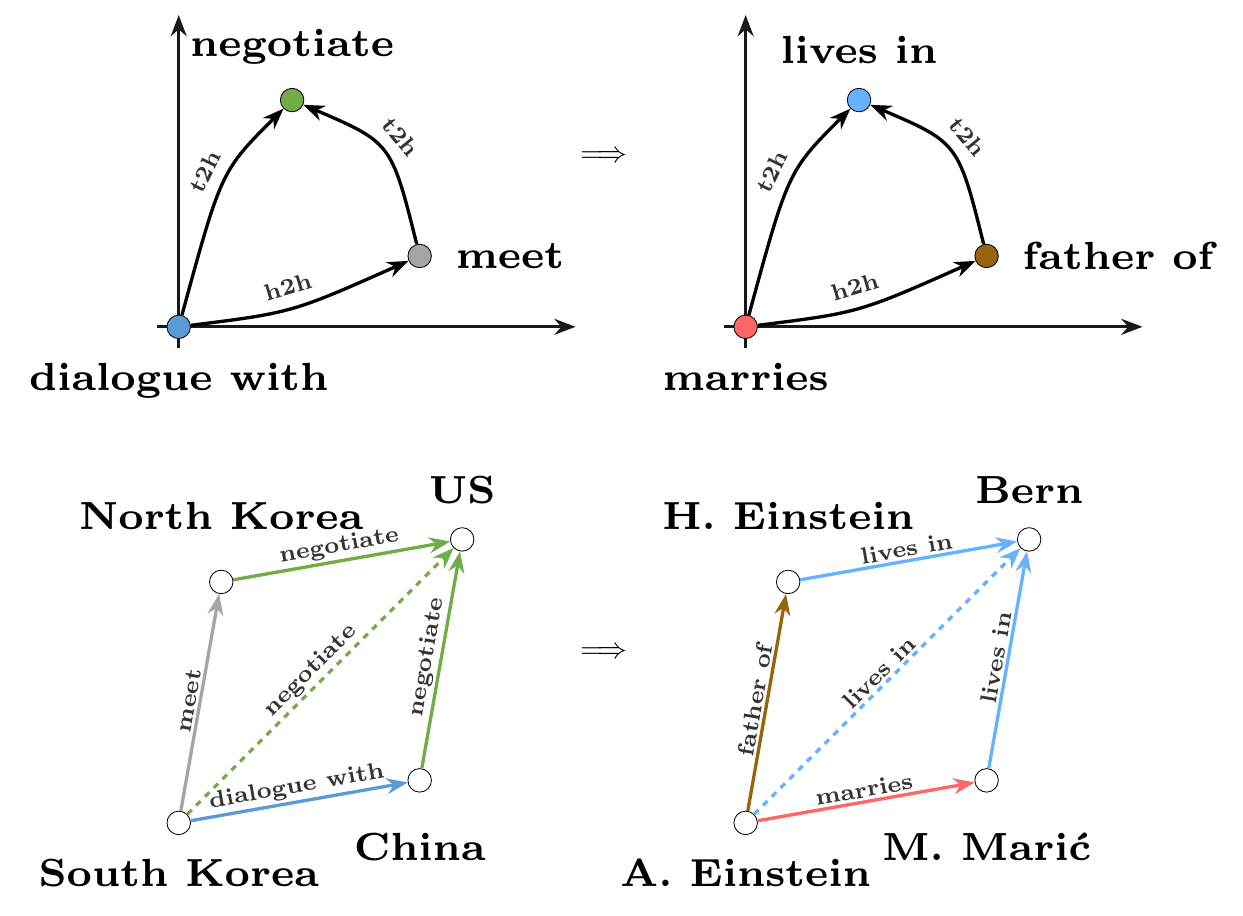}
        \caption{}
    \end{subfigure}
    \hfill
    \begin{subfigure}[b]{0.32\textwidth}
        \includegraphics[width=\linewidth]{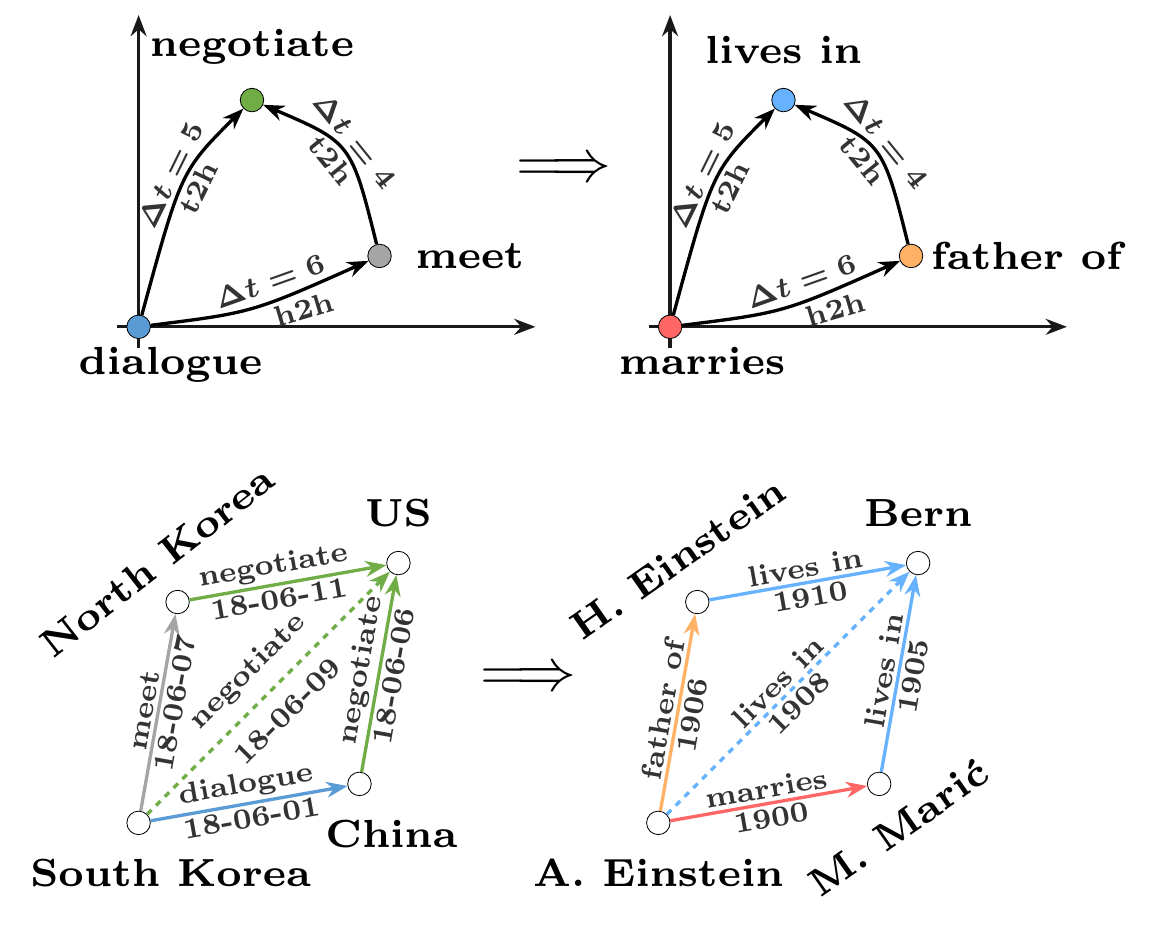}
        \caption{}
    \end{subfigure}
    \caption{Relative Representation Transfer}
    \label{fig:relative_representation}
\end{figure*}

\section{Hyperparameter and Training Details}
\label{sec:parameter}
Our experiments were conducted on 4 NVIDIA A100 GPUs with 48GB of RAM. We set the maximal training epoch as 10 and negative samples as 512. We use a batch size of 16, 2, and 1 for training ICEWS14, ICEWS05-15 and GDELT respectively. We set the dimension size $d$ as 64. We employed the AdamW optimizer and set the learning rate as 0.0005. More details can be seen in Table \ref{table:parameter}. For the baseline methods, we utilized the settings in their original papers. 

\begin{table*}
\centering
    \caption{
    Detailed hyperparameters. $GNN_r$ denotes the relation representation encoder, $GNN_q$ is quadruple representation encoder.
    }
    \label{table:parameter}
{ 
\begin{tabular}{lccccc}
\toprule
Module & Hyperparameter & Pre-training  \cr
\hline
    \multirow{4}{*}{$GNN_r$} & \# layers & 6 \cr
   &  hidden dim & 64 \cr
   & MSG & Dual \cr
   & AGG & sum \cr
   \hline
    \multirow{5}{*}{$GNN_e$} & \# layers & 6 \cr
   &  hidden dim & 64 \cr
   & T-MSG & TComplEx \cr
   & AGG & sum \cr
   & $k$ & ICEWS14:0; ICEWS0515, GDELT:1 \cr
   & $\beta$ & 10,000 \cr
\hline   
  \multirow{2}{*}{Prediction} & $f_{\theta}$ & 2-layer MLP \cr
  & $\alpha$ & ICEWS14:0.5; ICEWS0515, GDELT:0.8 \cr
\bottomrule
\end{tabular}
}
\end{table*}

\section{Hyperparameter Analysis}
We conduct a sensitivity analysis of three key hyperparameters, $k$, $\alpha$, and $\beta$ on model performance by pre-training on the ICEWS14 dataset. Figure \ref{fig:hyperparamter_study} presents the MRR and H@10 results under different values of each hyperparameter while keeping all others fixed at their optimal settings.

\textbf{Effect of $k$.}  
As shown in Figure \ref{fig:hyperparamter_study}(a), increasing $k$, which controls the length of the local entity graph, leads to a gradual decline in both MRR and H@10. As ICEWS14 is a news dataset and contains mainly short-term temporal relations, deeper propagation may  cause oversmoothing and decrease the model's performance.

\textbf{Effect of $\alpha$.}  
Figure~\ref{fig:hyperparamter_study}(b) shows that performance is highly sensitive to the choice of $\alpha$, which balances the contribution of local and global temporal contexts. Both MRR and H@10 significantly improve as $\alpha$ increases from 0 to 0.5, indicating the benefit of incorporating both types of temporal information. However, performance degrades when $\alpha$ approaches 1, implying that overemphasis on either context harms generalization. The results highlight the importance of maintaining a moderate balance between local and global temporal patterns.

\textbf{Effect of $\beta$.}  
As illustrated in Figure \ref{fig:hyperparamter_study}(c), the model is highly robust to the choice of $\beta$, which controls the sinusoidal frequency of temporal embeddings. Across a wide range of values—from $10^2$ to $10^6$—MRR and H@10 remain largely stable, with negligible performance fluctuations. 

 
The hyperparameter analysis reveals that $\alpha$ is the most critical parameter and requires careful tuning to balance local and global temporal information. The choice of $k$ should be dataset-dependent; for datasets characterized by short-term temporal patterns, smaller values of $k$ are preferable. In contrast, the model is relative insensitive to $\beta$, so we set it as 10,000 across all datasets.

\begin{figure*}
    \centering
    \begin{subfigure}[b]{0.32\textwidth}
        \includegraphics[width=\linewidth]{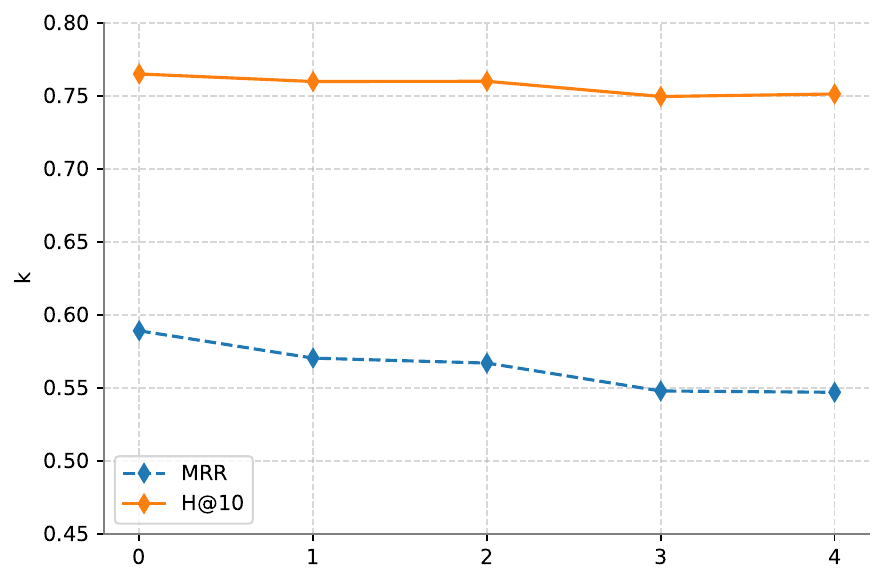}
        \caption{$k$'s performance}
    \end{subfigure}
    \hfill
    \begin{subfigure}[b]{0.32\textwidth}
        \includegraphics[width=\linewidth]{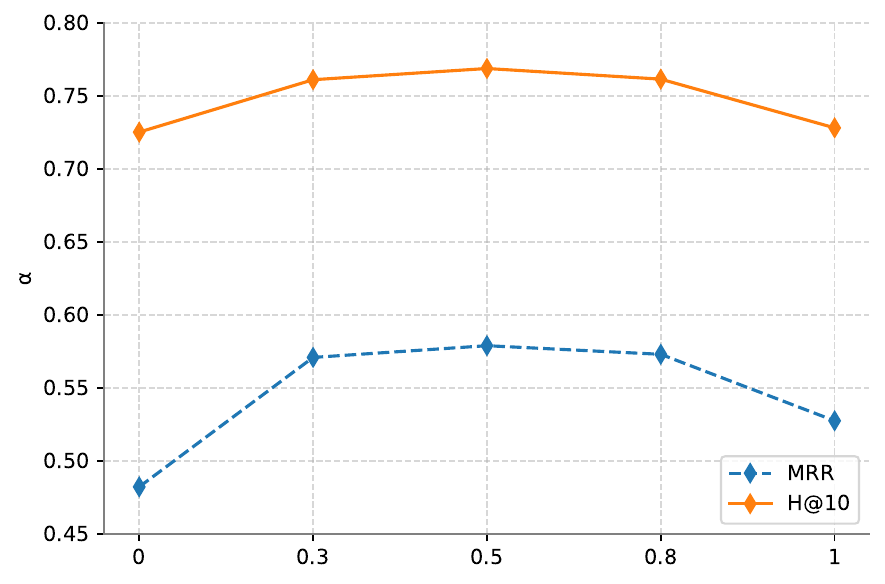}
        \caption{$\alpha$'s performance}
    \end{subfigure}
    \hfill
    \begin{subfigure}[b]{0.32\textwidth}
        \includegraphics[width=\linewidth]{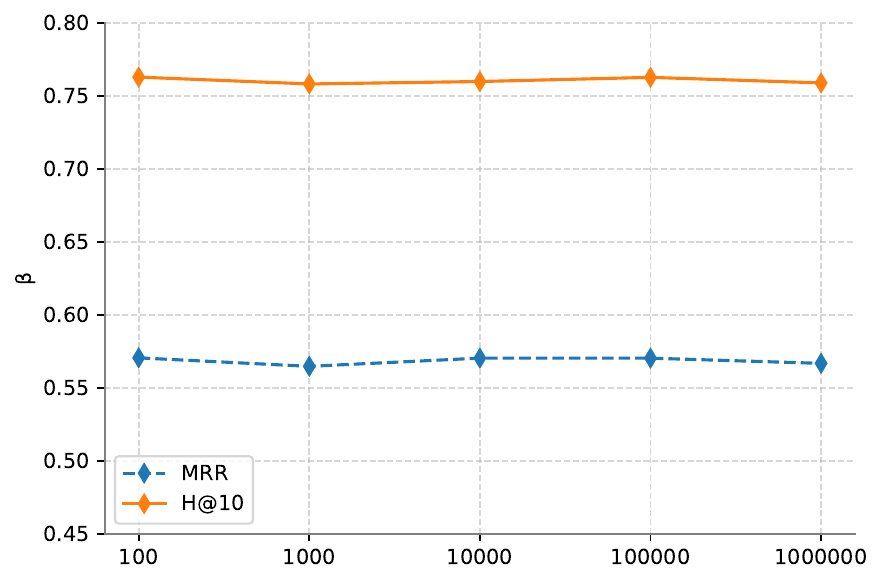}
        \caption{$\beta$'s performance}
    \end{subfigure}
    \caption{MRR and H@10 performance with different $k$, $\alpha$ and $\beta$. We report ICEWS14's pre-training performance and keep all other hyper-parameters as the fixed best setting.}
    \label{fig:hyperparamter_study}
\end{figure*}

\section{Case Study}
\label{sec:case_study}
\begin{figure*}
    \centering
    \begin{subcaptionbox}{Global Quadruple Representation \label{fig:first}}[0.48\textwidth]
        {\includegraphics[width=\linewidth]{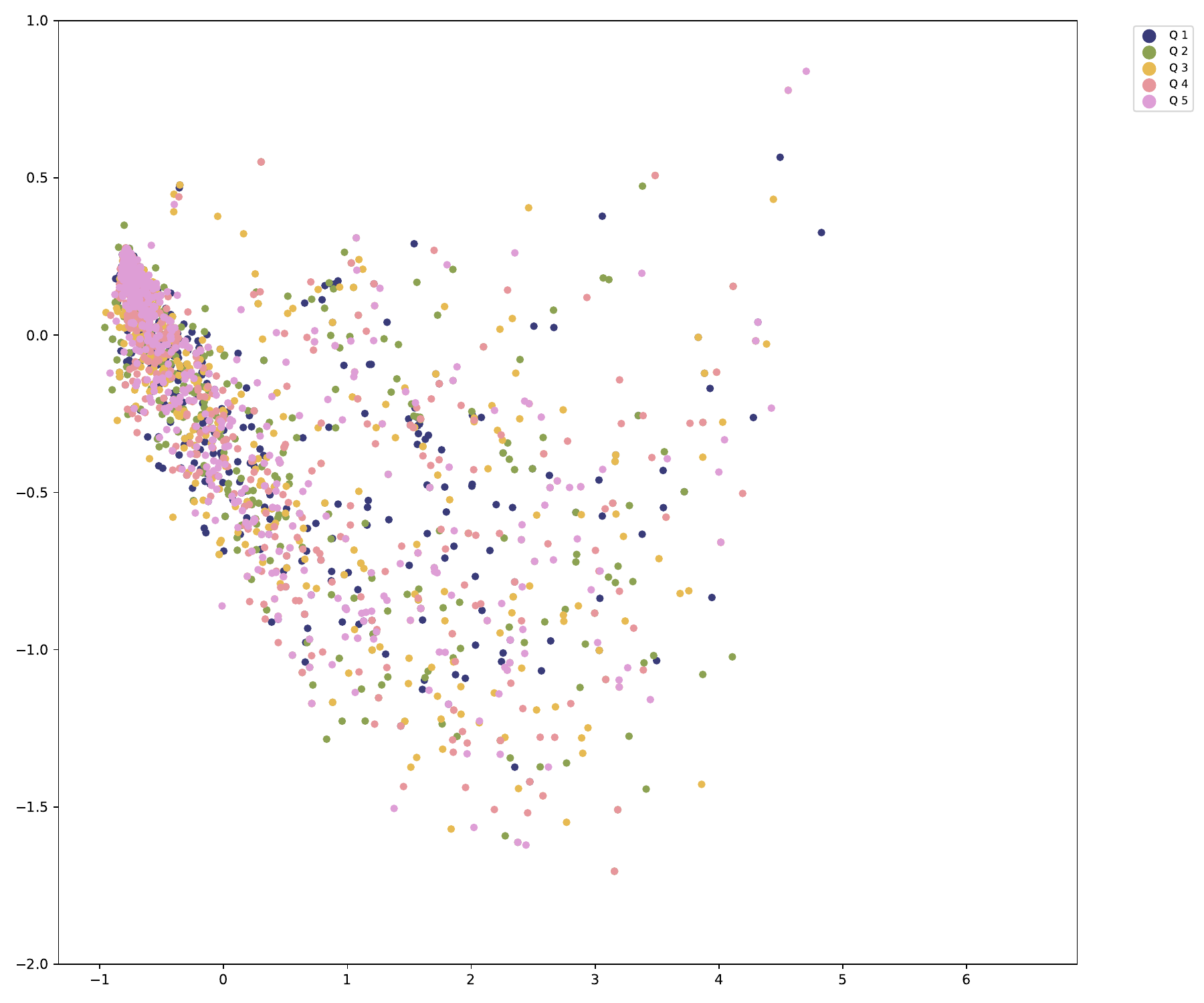}}
    \end{subcaptionbox}
    \hfill
    \begin{subcaptionbox}{Local Quadruple Representation\label{fig:second}}[0.48\textwidth]
        {\includegraphics[width=\linewidth]{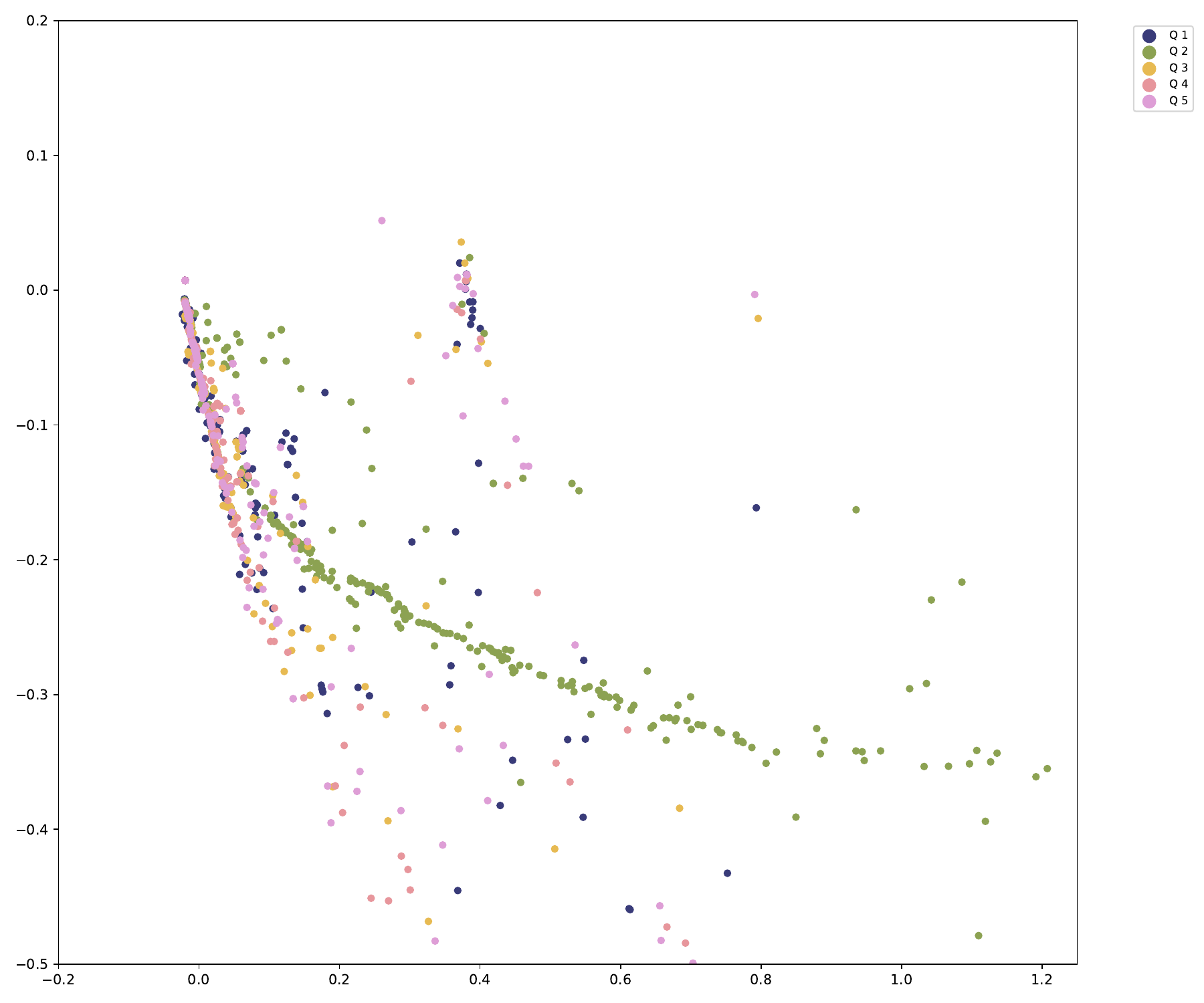}}
    \end{subcaptionbox}
    \caption{PCA visualizations of five quadruples with (China, Host a visit, ?, ?)}
    \label{fig:case_study}
\end{figure*}

To intuitively demonstrate \ModelName{}’s capability in generating compact and distinguishable embeddings, we employ PCA to visualize the distribution of embeddings for five representative quadruples: Q1: (China, Host a Visit, Macky Sall, 2014-02-19), Q2:(China, Host a Visit, Teo Chee Hean, 2014-10-28), Q3: (China, Host a Visit, Chuck Hagel, 2014-04-08), Q4: (China, Host a Visit, Sar Kheng, 2014-06-03), and Q5: (China, Host a Visit, Barack Obama, 2014-12-25), as illustrated in Figure \ref{fig:case_study}. The visualization reveals that the Global Quadruple Representation tends to capture general relational patterns, producing embeddings that are closely clustered for the same query structure (China, Host a Visit, ?, ?). In contrast, the Local Quadruple Representation is more sensitive to immediate temporal context, resulting in more distinctive embeddings across different time points. This dual representation mechanism highlights \ModelName{}’s strength in capturing both global structural semantics and local temporal dynamics. For example, for Q5, Barack Obama is frequently the tail entity for similar queries, leading the global encoder to assign him a high probability and correctly predict the answer. Meanwhile, in Q1, a supporting local event (Macky Sall, Make a Visit, China, 2014-02-18) appears in the surrounding context, enabling the local quadruple encoder to identify Macky Sall correctly. However, the absence of any relevant local events involving Teo Chee Hean in the Q2 causes the model to mispredict Barack Obama.

\subsection{Strict Recurrency Transfer}
\label{app:strict_recurrency}

To provide concrete evidence for the claim that temporal encoding is what transfers the \textit{Strict Recurrency} pattern across domains (Section~\ref{sec:result}), we isolate the Strict Recurrent subset of the YAGO test set and compare \ModelName{} against ULTRA when both are trained on GDELT. Although only $2.2\%$ of GDELT training queries are Strict Recurrent, this subset accounts for $92.7\%$ of YAGO test queries ($37{,}138$ queries). As shown in Table~\ref{tbl:strict_recurrency}, \ModelName{} substantially outperforms ULTRA on this subset, with a $+25.2$ gain in Hits@10. Because ULTRA lacks temporal encoding, it cannot recover the recurrency signal when the source graph provides few recurrent examples, whereas \ModelName{}'s relative temporal ordering transfers the pattern directly.

\begin{table}[h]
    \centering
    \caption{Performance on the YAGO Strict Recurrent subset ($92.7\%$ of YAGO test queries), with both models trained on GDELT. Best results in \textbf{bold}.}
    \label{tbl:strict_recurrency}
    \begin{tabular}{lcccc}
        \toprule
        Model & Subset & $n$ & MRR & H@10 \\
        \midrule
        ULTRA        & Strict Recurrent & $37{,}138$ & $52.3$          & $68.1$ \\
        \ModelName{} & Strict Recurrent & $37{,}138$ & $\textbf{68.1}$ & $\textbf{93.3}$ \\
        \bottomrule
    \end{tabular}
\end{table}

\subsection{Qualitative Error Analysis}
\label{app:error_analysis}

To understand \emph{where} \ModelName{}'s gains come from, we perform a per-query analysis of a model trained on ICEWS05-15 and tested on ICEWS14, comparing against ULTRA.

First, grouping queries by relation (Table~\ref{tbl:error_by_relation}) shows that the improvement is concentrated on relations whose events cluster in short time windows: \textit{Host a visit} and \textit{Consult} gain $+40.1$ and $+32.5$ MRR respectively, whereas a temporally uniform relation such as \textit{Make statement} sees only a marginal $+4.4$ gain. This matches the intuition that temporal encoding helps most when the timing of events is informative.

\begin{table}[h]
    \centering
    \caption{Per-relation MRR (trained on ICEWS05-15, tested on ICEWS14). $\Delta$MRR is \ModelName{} minus ULTRA.}
    \label{tbl:error_by_relation}
    \begin{tabular}{lcccc}
        \toprule
        Relation & $n$ & ULTRA & \ModelName{} & $\Delta$MRR \\
        \midrule
        Host a visit   & $432$  & $37.1$ & $77.2$ & $+40.1$ \\
        Consult        & $977$  & $49.6$ & $82.0$ & $+32.5$ \\
        Make statement & $1500$ & $46.3$ & $50.7$ & $+4.4$  \\
        \bottomrule
    \end{tabular}
\end{table}

Second, grouping queries by novelty (Table~\ref{tbl:error_by_novelty}) shows that ULTRA already reaches $93.9$ Hits@10 on recurrent triples, leaving little headroom there. \ModelName{}'s gains therefore concentrate on \emph{novel} triples, whose $(h, r)$ pair has recent activity but whose exact $(h, r, t)$ combination has not been seen: here \ModelName{} improves MRR by $+14.8$, versus $+10.5$ on recurrent triples. Temporal encoding lets the model exploit recent context to resolve genuinely new facts rather than merely repeating history.

\begin{table}[h]
    \centering
    \caption{MRR by query novelty (trained on ICEWS05-15, tested on ICEWS14). $\Delta$MRR is \ModelName{} minus ULTRA.}
    \label{tbl:error_by_novelty}
    \begin{tabular}{lcccc}
        \toprule
        Query type & $n$ & ULTRA & \ModelName{} & $\Delta$MRR \\
        \midrule
        Novel $(h,r,t)$     & $5097$ & $35.5$ & $50.3$ & $+14.8$ \\
        Recurrent $(h,r,t)$ & $3866$ & $71.4$ & $81.9$ & $+10.5$ \\
        \bottomrule
    \end{tabular}
\end{table}

\section{Theoretical Analysis}
\label{sec:theoretical_analysis}
In this section, we provide a theoretical analysis demonstrating the model’s ability to transfer across time and effectively capture a variety of periodic patterns.

\paragraph{Temporal Transferability}
As illustrated in Figure \ref{fig:relative_representation}, the temporal difference is the transferable information in TKGs. Therefore, we expect that such transferability must be followed in the time embedding space. 
Here, we aim to show that if four time points' indices—$\tau_1, \tau_2$ in the training graph and $\tau'_1, \tau'_2$ in the test graph—hold $\Delta T = \tau_2 - \tau_1 = \Delta T' =  \tau'_2 - \tau'_1$, then we expect that such transferability is preserved via sinusoidal positional encoding in the form of Euclidean distance
\(\|\mathrm{TE}(\tau_2)-\mathrm{TE}(\tau_1)\| = \|\mathrm{TE}(\tau'_2)-\mathrm{TE}(\tau'_1)\|\)
that is governed solely by their time difference $\Delta T$, rather than the
absolute times.  

The following theorem states and proves this: 

\begin{nbtheorem}[Time-Shift Invariance in Sinusoidal Positional Embeddings]
Let \(d\in\mathbb{N}\) be \emph{even} and fix positive frequencies
\(\omega_0,\omega_1,\dots,\omega_{\frac{d}{2}-1}>0\).
Define the \emph{sinusoidal temporal embedding}
\(\mathrm{TE}:\mathbb{N}\to\mathbb{R}^{d}\) component-wise by
\[
  \bigl[\mathrm{TE}(t)\bigr]_{2n}   \;=\;\sin\!\bigl(\omega_n\,t\bigr),
  \quad
  \bigl[\mathrm{TE}(t)\bigr]_{2n+1} \;=\;\cos\!\bigl(\omega_n\,t\bigr),
  \quad 0\le n<\tfrac{d}{2}.
\]

For any four time-points indices \(\tau_1,\tau_2,\tau'_1,\tau'_2\in\mathbb{N}\) set
\[
  \Delta_\tau \;:=\; \tau_2-\tau_1,
  \qquad
  \Delta_\tau' \;:=\; \tau'_2-\tau'_1.
\]
If \(\Delta_\tau=\Delta_\tau'\) then
\[
  \bigl\|\mathrm{TE}(\tau_2)-\mathrm{TE}(\tau_1)\bigr\|
  \;=\;
  \bigl\|\mathrm{TE}(\tau'_2)-\mathrm{TE}(\tau'_1)\bigr\|.
\]
That is, the Euclidean distance between two sinusoidal embeddings
depends \emph{only} on the time difference, not on the absolute
timestamps.
\end{nbtheorem}

\begin{proof}[Step-by-step proof]

\smallskip
\hrule\smallskip
\noindent

Fix an index \(n\in\{0,\dots,\tfrac{d}{2}-1\}\) and two times
\(t_1,t_2\in\mathbb{N}\).  Denote
\[
  \Delta = t_2-t_1,
  \qquad
  \Sigma = t_1+t_2.
\]
Using the standard sum-and-difference identities,
\begin{align*}
  \sin(\omega_n t_2)-\sin(\omega_n t_1)
  &= 2\,
     \sin\!\Bigl(\tfrac{\omega_n\Delta}{2}\Bigr)\,
     \cos\!\Bigl(\tfrac{\omega_n\Sigma}{2}\Bigr), \\[6pt]
  \cos(\omega_n t_2)-\cos(\omega_n t_1)
  &= -\,2\,
     \sin\!\Bigl(\tfrac{\omega_n\Delta}{2}\Bigr)\,
     \sin\!\Bigl(\tfrac{\omega_n\Sigma}{2}\Bigr).
\end{align*}

Thus, the squared contribution of this frequency to the
Euclidean distance is
\[
  \bigl[\sin(\omega_n t_2)-\sin(\omega_n t_1)\bigr]^2
  + \bigl[\cos(\omega_n t_2)-\cos(\omega_n t_1)\bigr]^2
  \;=\;
  4\,\sin^2\!\Bigl(\tfrac{\omega_n\Delta}{2}\Bigr),
\]
because \(\cos^2+\sin^2=1\) removes any dependence on
\(\Sigma\).

Summing over all \(n\) gives
\[
  \bigl\|\mathrm{TE}(t_2)-\mathrm{TE}(t_1)\bigr\|^2
  \;=\;
  \sum_{n=0}^{\tfrac{d}{2}-1}
    4\,\sin^2\!\Bigl(\tfrac{\omega_n\Delta}{2}\Bigr)
  \;=\;
  4\sum_{n=0}^{\tfrac{d}{2}-1}
    \sin^2\!\Bigl(\tfrac{\omega_n\Delta}{2}\Bigr),
\]
which is a function of \(\Delta=t_2-t_1\) alone.
Consequently, if \(\tau_2-\tau_1=\tau'_2-\tau'_1\) then
\(\Delta_\tau=\Delta_\tau'\) and

\begin{align}
\bigl\|\mathrm{TE}(\tau_2) - \mathrm{TE}(\tau_1)\bigr\|^2 
&= 4 \sum_n \sin^2\!\left( \tfrac{\omega_n \Delta_\tau}{2} \right) \notag \\
&= 4 \sum_n \sin^2\!\left( \tfrac{\omega_n \Delta_\tau'}{2} \right) \notag \\
&= \bigl\|\mathrm{TE}(\tau_2') - \mathrm{TE}(\tau_1')\bigr\|^2. \tag{5}
\end{align}

Taking square roots results in the claimed equality of norms.
\end{proof}

\paragraph{Capturing Diverse Periodicities and Frequencies}
Temporal facts expressed as $(s, p, o, \tau)$, where $\tau \in \mathbb{T}$ represents the time annotation, can capture periodic phenomena, such as the Olympic Games occurring every four years or annual events like the Nobel Prize ceremonies.
Modeling diverse periodic 
behaviors is crucial in the fully-inductive transfer learning, as it enables systems to accurately represent and analyze the dynamic evolution of diverse relationships and events. 
In this section, we prove that \ModelName{} is capable of capturing periodic (with different frequencies) 
temporal facts.   
In the following theorem, we prove that the scorer in Equation \ref{eq:scorefunctionequation} can capture and express various frequencies even if we assume there $f_{\theta}$ is simply a linear function (with $m$ linear nodes for representing $m$ different frequencies). For simplicity, we loosely use $\tau$ as the timestamp index $i$ in $\tau_i$ in the following.

\begin{nbtheorem}[Multi-Frequency Affine Scorer
]
\label{thm:multi-frequency}
Let
\begin{itemize}
    \item $T>1$ be an integer time horizon ;
    \item $d_{\mathrm{PE}}\ge T$ be an even positional–encoding dimension;
    \item $P_1,\dots ,P_m,$ such that $P_i\mid T$ for every $i=1,\dots ,m,$ is a family of positive integers.
    \item 
    $[\mathrm{TE}(\tau)]_{2n}   = \sin(\omega_n\tau),
            \qquad
            [\mathrm{TE}(\tau)]_{2n+1} = \cos(\omega_n\tau),
            \qquad
            0\le n < \tfrac{d_{\mathrm{PE}}}{2},$
             with arbitrary real frequencies $\omega_0,\dots ,\omega_{d_{\mathrm{PE}}/2-1}$, be (standard) sinusoidal positional encoding defined for any $\tau$.

    \item $G = \{V_1,\dots,V_m\}\subset\mathbb{R}^{d_V}$ be $m$ fixed context vectors and,
    \item  $g_i : \{0,1,\dots ,T-1\}\longrightarrow \mathbb{R}$
 be a \emph{non-constant} sequence for each $i=1,\dots ,m$, that we extend to all $\mathbb{Z}$ by $P_i$-periodicity:
\[
  g_i(\tau) := g_i\!\bigl(\tau\bmod P_i\bigr),
  \qquad \tau\in\mathbb{Z}.
\]

\begin{align}
M &= 
\begin{pmatrix}
\mathrm{TE}(0)^\top \\
\mathrm{TE}(1)^\top \\
\vdots \\
\mathrm{TE}(T\!-\!1)^\top
\end{pmatrix}
\in \mathbb{R}^{T \times d_{\mathrm{PE}}}, \notag \\[6pt]
B^{(i)} &= 
\begin{pmatrix}
\mathrm{TE}(P_i)^\top    - \mathrm{TE}(0)^\top \\
\mathrm{TE}(P_i\!+\!1)^\top  - \mathrm{TE}(1)^\top \\
\vdots \\
\mathrm{TE}(T\!-\!1\!+\!P_i)^\top - \mathrm{TE}(T\!-\!1)^\top
\end{pmatrix}
\in \mathbb{R}^{T \times d_{\mathrm{PE}}}. \notag
\end{align}

be encoding matrices, we stack $B^{(i)}$ block-diagonally:
\[
  \widetilde B := \operatorname{diag}\!\bigl(B^{(1)},\dots ,B^{(m)}\bigr)
  \in\mathbb{R}^{mT\times m\,d_{\mathrm{PE}}}.
\]

\end{itemize}

\medskip
\noindent

\medskip
\noindent

\medskip
\noindent

\bigskip
\noindent
Define
\[
  g:=
  \begin{pmatrix}
    g_1\\
    \vdots\\
    g_m
  \end{pmatrix}
  \in\mathbb{R}^{mT},
  \quad
  w_{\mathrm{PE}}:=
  \begin{pmatrix}
    w_{\mathrm{PE}}^{(1)}\\
    \vdots\\
    w_{\mathrm{PE}}^{(m)}
  \end{pmatrix}
  \in\mathbb{R}^{m\,d_{\mathrm{PE}}}.
\]

\begin{enumerate}[label=\textup{(\Alph*)},leftmargin=*,itemsep=4pt]
  \item \textbf{Compatibility \((\mathcal{C})\).}\;  
        There exists $w_{\mathrm{PE}}\in\mathbb{R}^{m\,d_{\mathrm{PE}}}$ solving the linear system
        \[
          \boxed{\;
          \begin{pmatrix}
            I_m\!\otimes\! M \\[4pt]
            \widetilde B
          \end{pmatrix}
          w_{\mathrm{PE}}
          =
          \begin{pmatrix}
            g\\[3pt]0
          \end{pmatrix}}
          \tag{$\mathcal{C}$}
        \]
        where $I_m\otimes M$ is the block-diagonal Kronecker product whose $i$-th diagonal block equals $M$.
  \item \textbf{Existence of a scorer \((\mathcal{S})\).}\;  
        There exist parameters
        \[
          \theta
          \;=\;
          \bigl(W_{\mathrm{PE}},\,w_V,\,b\bigr)
          \in
          \mathbb{R}^{d_{\mathrm{PE}}\times m}
          \times\mathbb{R}^{d_V}
          \times\mathbb{R}
        \]
        such that the affine map
        \[
          f_\theta\bigl(V,\mathrm{TE}(\tau)\bigr)
          :=
          w_V^\top V
          \;+\;
          W_{\mathrm{PE}}^\top\,\mathrm{TE}(\tau)
          \;+\;
          b
          \;\in\mathbb{R}^{m}
        \]
        obeys, for each $i=1,\dots ,m$,
        \begin{enumerate}[label=\textup{(\roman*)}]
          \item $P_i$-periodicity in $\tau\in\mathbb{Z}$;
          \item non-constancy as a function of $\tau$;
          \item interpolation on the basic window:
                \[
                  \bigl[f_\theta\!\bigl(V_i,\mathrm{TE}(\tau)\bigr)\bigr]_i
                  \;=\;
                  w_V^\top V_i + g_i(\tau),
                  \qquad
                  0\le\tau<T.
                \]
        \end{enumerate}
\end{enumerate}

Given $\mathcal{C}$ and $\mathcal{S}$, the following holds
\begin{center}
\framebox{\(\mathcal{C}\;\Longleftrightarrow\;\mathcal{S}\)}.
\end{center}

\bigskip
\noindent
When these equivalent conditions hold one may choose \emph{any} solution $w_{\mathrm{PE}}$ of~\((\mathcal{C})\), set
\[
  W_{\mathrm{PE}}
  :=
  \bigl[w_{\mathrm{PE}}^{(1)}\;\cdots\;w_{\mathrm{PE}}^{(m)}\bigr],
  \qquad
  b=0,
\]
and pick an arbitrary $w_V$; the resulting $\theta$ satisfies \(\mathcal{S}\).
\end{nbtheorem}

\begin{proof}[Step-by-step proof of \(\mathcal{C}\Longleftrightarrow\mathcal{S}\)]
We write $w_{\mathrm{PE}}=(w_{\mathrm{PE}}^{(1)},\dots ,w_{\mathrm{PE}}^{(m)})$ with
each block $w_{\mathrm{PE}}^{(i)}\in\mathbb{R}^{d_{\mathrm{PE}}}$.

\smallskip
\hrule\smallskip
\noindent
\textbf{Part I: \(\mathcal{C}\;\Longrightarrow\;\mathcal{S}\).}
Assume a vector $w_{\mathrm{PE}}$ satisfies system~\((\mathcal{C})\). 
Here, we show that if the compatibility condition holds, then a scorer exists. We divide the proof into four parts as follows:

\begin{enumerate}[label=\textbf{Step~\arabic*:},leftmargin=*,itemsep=6pt]
\item The top block $(I_m\otimes M)w_{\mathrm{PE}}=g$ is equivalent to
      \[
        M\,w_{\mathrm{PE}}^{(i)} \;=\; g_i \Big|_{\,0\le\tau<T},
        \quad
        i=1,\dots ,m.
      \tag{1}
      \]
\item 
      The block $\widetilde B\,w_{\mathrm{PE}}=0$ gives for every $i$
      \[
        B^{(i)}\,w_{\mathrm{PE}}^{(i)} = 0
        \;\Longrightarrow\;
        w_{\mathrm{PE}}^{(i)\top}
        \bigl(\mathrm{TE}(\tau+P_i)-\mathrm{TE}(\tau)\bigr)=0, 
        \quad
        0\le\tau<T.
      \]
      Thus
      \[
        w_{\mathrm{PE}}^{(i)\top}\mathrm{TE}(\tau+P_i)
        =w_{\mathrm{PE}}^{(i)\top}\mathrm{TE}(\tau),
        \quad 0\le\tau<T.
        \tag{2}
      \]
      Replacing $\tau$ by $\tau+P_i$ and iterating $k$ times shows
      \[
        w_{\mathrm{PE}}^{(i)\top}\mathrm{TE}(\tau+kP_i)
        = w_{\mathrm{PE}}^{(i)\top}\mathrm{TE}(\tau),
        \quad
        \forall\tau\in\{0,\dots ,T-1\},\;k\in\mathbb{Z},
      \]
      i.e., $P_i$-periodicity on \emph{all} integers.
\item 
      Combine~(1) with periodicity: for every $\tau\in\mathbb{Z}$
      \[
        w_{\mathrm{PE}}^{(i)\top}\mathrm{TE}(\tau)
        = g_i\!\bigl(\tau\bmod P_i\bigr).
        \tag{3}
      \]
\item 
      Define
      \[
        W_{\mathrm{PE}}
        :=
        \bigl[w_{\mathrm{PE}}^{(1)}\;\cdots\;w_{\mathrm{PE}}^{(m)}\bigr],
        \quad
        b:=0,
        \quad
        \text{choose any }w_V\in\mathbb{R}^{d_V}.
      \]
      Then
      \[
        f_\theta\bigl(V,\mathrm{TE}(\tau)\bigr)
        =
        w_V^\top V
        +
        \begin{pmatrix}
          w_{\mathrm{PE}}^{(1)\top}\mathrm{TE}(\tau)\\[-1pt]
          \vdots\\[-1pt]
          w_{\mathrm{PE}}^{(m)\top}\mathrm{TE}(\tau)
        \end{pmatrix},
      \]
      whose $i$-th component equals the right-hand side of~(3).  
      Therefore:
      a) it is $P_i$-periodic by construction;  
      b) it is non-constant because each $g_i$ is non-constant;  
      c) on $0\le\tau<T$, (3) reduces to \(w_{\!V}^{\top}V_i+g_i(\tau)\).
      Hence \(\mathcal{S}\) holds.
\end{enumerate}

\smallskip
\hrule\smallskip
\noindent
\textbf{Part II: \(\mathcal{S}\;\Longrightarrow\;\mathcal{C}\).}

Conversely, suppose parameters \(\theta=(W_{\mathrm{PE}},w_V,b)\) satisfy \(\mathcal{S}\).
Write the columns as
\(
  W_{\mathrm{PE}}=[\,w_{\mathrm{PE}}^{(1)}\mid\cdots\mid w_{\mathrm{PE}}^{(m)}].
\)

\begin{enumerate}[label=\textbf{Step~\arabic*:},leftmargin=*,itemsep=6pt]
\item 
      Evaluating property~(iii) at $0\le\tau<T$ gives
      \[
        M\,w_{\mathrm{PE}}^{(i)} = g_i \Big|_{\,0\le\tau<T},
        \quad
        i=1,\dots ,m.
        \tag{4}
      \]
\item \emph{$P_i$-periodicity implies the $B^{(i)}$ equations.}\;
      Property (i) yields, for every integer $\tau$,
      \(
        w_{\mathrm{PE}}^{(i)\top}\mathrm{TE}(\tau+P_i)=
        w_{\mathrm{PE}}^{(i)\top}\mathrm{TE}(\tau)
      \).
      Specialising to $\tau=0,\dots ,T-1$ one obtains
      \(
        B^{(i)}\,w_{\mathrm{PE}}^{(i)}=0
      \).
\item 
      Collecting (4) for all $i$ and the $B^{(i)}$ equations in a single vector $w_{\mathrm{PE}}$ yields exactly system~\((\mathcal{C})\).  Thus \(\mathcal{C}\) holds.
\end{enumerate}

\smallskip
\hrule\smallskip
\noindent
Both directions are proven; therefore \(\mathcal{C}\Longleftrightarrow\mathcal{S}\).  
The final comment about choosing $(W_{\mathrm{PE}},w_V,b)$ is precisely the construction in Part I, Step 4.
\end{proof}

Following Theorem \ref{thm:multi-frequency}, we are interested in knowing under which conditions the compatibility condition holds itself, i.e., when the system has a solution.

\begin{nbtheorem}[Universal Compatibility under harmonically–aligned frequencies]%
\label{thm:compat-always}
Fix
\[
  T>1,\qquad d_{\mathrm{PE}}\ge T,\qquad
  P_1,\dots ,P_m\in\mathbb N,\;\;P_i\mid T.
\]

\medskip
\noindent
Let \(L:=\operatorname{lcm}(P_1,\dots ,P_m)\)\footnote{least common multiple} and choose \emph{any} collection of
integers \(k_0,k_1,\dots ,k_{d_{\mathrm{PE}}/2-1}\).
Define the frequencies
\[
  \omega_n \;:=\; \frac{2\pi k_n}{L},
  \qquad 0\le n<\tfrac{d_{\mathrm{PE}}}{2},
\]
and build the sinusoidal positional encoding
\[
  [\mathrm{TE}(\tau)]_{2n}   = \sin(\omega_n\tau),
  \quad
  [\mathrm{TE}(\tau)]_{2n+1} = \cos(\omega_n\tau),
  \qquad \tau\in\mathbb N.
\]

Form the matrices as follows
\begin{align}
M &= 
\begin{pmatrix}
\mathrm{TE}(0)^\top \\
\vdots \\
\mathrm{TE}(T{-}1)^\top
\end{pmatrix}
\in \mathbb{R}^{T \times d_{\mathrm{PE}}}, \tag{9} \\[6pt]
B^{(i)} &= 
\begin{pmatrix}
\mathrm{TE}(P_i)^\top - \mathrm{TE}(0)^\top \\
\vdots \\
\mathrm{TE}(T{-}1{+}P_i)^\top - \mathrm{TE}(T{-}1)^\top
\end{pmatrix}
\in \mathbb{R}^{T \times d_{\mathrm{PE}}}. \tag{10}
\end{align}

For these frequencies, every \(B^{(i)}\) vanishes identically; hence, the stacked
compatibility system
\[
  \begin{pmatrix}
    I_m\otimes M \\[4pt] \mathrm{diag}(B^{(1)},\dots ,B^{(m)})
  \end{pmatrix}
  w_{\mathrm{PE}}
  =
  \begin{pmatrix}
    g\\[3pt]0
  \end{pmatrix}
  \tag{$\mathcal{C}$}
\]
is always solvable, regardless of the choice of \emph{any} target sequences
\(g_i\colon\{0,\dots ,T-1\}\to\mathbb R\).
\end{nbtheorem}

\smallskip
\hrule\smallskip
\noindent

\begin{proof}[We present the Step-by-step proof as follows.]

\textbf{Step 1.  Encoding is \(L\)-periodic.}
Because \(\omega_n L=2\pi k_n\),
\[
  \sin(\omega_n(\tau+L))=\sin(\omega_n\tau),\quad
  \cos(\omega_n(\tau+L))=\cos(\omega_n\tau),
\]
\(\forall\tau\in\mathbb Z\), so \(\mathrm{TE}(\tau+L)=\mathrm{TE}(\tau)\).

\smallskip
\textbf{Step 2.  Encoding is also \(P_i\)-periodic.}
Each \(P_i\) divides \(L\), hence \(\mathrm{TE}(\tau+P_i)=\mathrm{TE}(\tau)\).
Therefore
\[
  B^{(i)}=0\in\mathbb R^{T\times d_{\mathrm{PE}}},
  \qquad i=1,\dots ,m.
\]

\smallskip
\textbf{Step 3.  Compatibility degenerates to a single block.}
With every \(B^{(i)}\) equal to zero, system \((\mathcal{C})\) becomes
\[
  (I_m\otimes M)\,w_{\mathrm{PE}} \;=\; g .
  \tag{$\mathcal{C}'$}
\]

\smallskip
\textbf{Step 4.  Solve the reduced system.}
Write \(w_{\mathrm{PE}}=(w_{\mathrm{PE}}^{(1)},\dots ,w_{\mathrm{PE}}^{(m)})\)
and \(g=(g_1,\dots ,g_m)\) with blocks in \(\mathbb R^{T}\).
The Kronecker structure of \(I_m\otimes M\) splits \((\mathcal{C}')\) into
\(m\) independent systems
\[
  M\,w_{\mathrm{PE}}^{(i)} = g_i,\qquad i=1,\dots ,m.
  \tag{$\mathcal{C}_i$}
\]
Because \(d_{\mathrm{PE}}\ge T\), the rows of \(M\) are linearly dependent
at worst; pick any right inverse \(M^{+}\) (for instance the Moore–Penrose
pseudoinverse).  Then
\[
  w_{\mathrm{PE}}^{(i)} := M^{+}g_i
  \quad (i=1,\dots ,m)
\]
solves each \((\mathcal{C}_i)\), and hence \(w_{\mathrm{PE}}\) solves the full
system \((\mathcal{C})\).

\smallskip
\textbf{Step 5.  Conclusion.}
Compatibility holds \emph{for every choice of data} \(\{g_i\}\)
once the frequencies obey \(\omega_n=\tfrac{2\pi k_n}{L}\).
\end{proof}

\clearpage

\end{document}